\documentclass[twoside,11pt,preprint]{article}

\usepackage{blindtext}

\usepackage{jmlr2e}

\usepackage{paralist,amsmath, amssymb}
\usepackage{algorithm}
\usepackage{algorithmic}
\usepackage{array}
\def \A {\mathcal{A}}
\def \C {\mathcal{C}}

\def \K {\mathcal{K}}

\def \H {\mathcal{H}}
\def \F {\mathcal{F}}
\def \T {\mathcal{T}}
\def \E {\mathbb{E}}
\def \U {\mathcal{U}}

\def \u {\mathbf{u}}
\def \x {\mathbf{x}}

\def \c {\mathbf{c}}
\def \s {\mathbf{s}}
\def \w {\mathbf{w}}
\def \ze {\mathbf{0}}

\def \g {\mathbf{g}}
\def \v {\mathbf{v}}
\def \y {\mathbf{y}}

\DeclareMathOperator*{\reg}{R}
\DeclareMathOperator*{\Dreg}{DR}

\DeclareMathOperator*{\argmin}{argmin}

\newtheorem{thm}{Theorem}
\newtheorem{lem}{Lemma}

\newtheorem{assum}{Assumption}

\newtheorem{cor}{Corollary}

\usepackage{times}

\usepackage{lastpage}
\jmlrheading{xx}{2025}{1-\pageref{LastPage}}{xx/xx}{xx/xx}{xx-xxxx}{Yuanyu Wan, Chang Yao, Yitao Ma, Mingli Song, and Lijun Zhang}

\ShortHeadings{Non-stationary Delayed OCO: From Full-information to Bandit Setting}{Wan, Yao, Ma, Song, and Zhang}
\firstpageno{1}

\begin{document}

\title{Non-stationary Delayed Online Convex Optimization: From Full-information to Bandit Setting}

\author{%
 \name Yuanyu Wan \email wanyy@zju.edu.cn
 \AND
 \name Chang Yao \email changy@zju.edu.cn
 \AND
\name Yitao Ma \email mayt@zju.edu.cn
 \AND
 \name Mingli Song \email brooksong@zju.edu.cn\\
 \addr School of Software Technology, Zhejiang University, Ningbo, China\\
 \addr State Key Laboratory of Blockchain and Data Security, Zhejiang University, Hangzhou, China
 \AND
  \name Lijun Zhang \email zhanglj@lamda.nju.edu.cn\\
 \addr National Key Laboratory for Novel Software Technology, Nanjing University, Nanjing, China\\
 \addr School of Artificial Intelligence, Nanjing University, Nanjing, China%
}

\editor{}

\maketitle

\begin{abstract}
Although online convex optimization (OCO) under arbitrary delays has received increasing attention recently, previous studies focus on stationary environments with the goal of minimizing static regret. In this paper, we investigate the delayed OCO in non-stationary environments, and choose dynamic regret with respect to any sequence of comparators as the performance metric. To this end, we first propose an algorithm called Mild-OGD for the full-information case, where delayed gradients are available. The basic idea is to maintain multiple experts in parallel, each performing a gradient descent step with different learning rates for every delayed gradient according to their arrival order, and utilize a meta-algorithm to track the best one based on their delayed performance. Despite the simplicity of this idea, our novel analysis shows that the dynamic regret of Mild-OGD can be automatically bounded by $O(\sqrt{\bar{d}T(P_T+1)})$ under the in-order assumption and $O(\sqrt{dT(P_T+1)})$ in the worst case, where $\bar{d}$ and $d$ denote the average and maximum delay respectively, $T$ is the time horizon, and $P_T$ is the path-length of comparators. Moreover, we demonstrate that the result in the worst case is optimal by deriving a matching lower bound. Finally, we develop a bandit variant of Mild-OGD for a more challenging case with only delayed loss values. Interestingly, we prove that under a relatively large amount of delay, our bandit algorithm even enjoys the best dynamic regret bound of existing non-delayed bandit algorithms.
\end{abstract}

\section{Introduction}
Online convex optimization (OCO) has become a popular paradigm for solving sequential decision-making problems \citep{Online:suvery,Hazan2016,Survey-Orabona}. Specifically, according to the protocol of OCO, an online player acts as the decision maker that needs to choose a decision $\x_t$ from a convex set $\K\subseteq\mathbb{R}^n$ at each round $t\in[T]$. After the decision $\x_t$ is committed, the player suffers a loss $f_t(\x_t)$, where $f_t(\x):\mathbb{R}^n\mapsto\mathbb{R}$ is a convex function selected by an adversary. To improve the performance in subsequent rounds, the player needs to update the decision by exploiting information about loss functions in previous rounds. According to the type of available information, OCO can be divided into two categories: the full-information setting where the player can observe the entire loss function, and the bandit setting where only the incurred loss can be observed. Over the past decades, plenty of algorithms and theories have been introduced to guide the player for both full-information \citep{Zinkevich2003,Shai07,Hazan_2007} and bandit \citep{OBO05,Agarwal2010_COLT,Saha_2011} settings.


However, there exists a common assumption that the information about each function $f_t(\x)$ is revealed at the end of round $t$, which is not necessarily satisfied in many real applications. For example, in online advertisement \citep{McMahan2013,He2014_ADKDD}, each loss function depends on whether a user clicks an ad or not, which may not be decided even when the user has observed the ad for a long period of time. To tackle this issue, there has been a surge of research interest in OCO with arbitrarily delayed full-information \citep{Joulani13,Quanrud15,NeurIPS22-Wan,DOFW-Journal} and bandit feedback \citep{ICML20_Mertikopoulos,Bistritz-JMLR22,Wan-NIPS24}. Specifically, the corresponding information about $f_t(\x)$ is revealed at the end of round $t+d_t-1$, where $d_t\geq 1$ is an integer and denotes the delay. Unfortunately, these studies focus on developing algorithms to minimize the static regret of the player, i.e., $\reg(T)=\sum_{t=1}^Tf_t(\x_t)-\min_{\x\in\K}\sum_{t=1}^Tf_t(\x)$, which is only meaningful for stationary environments with at least one fixed decision minimizing the cumulative loss well. Thus, their algorithms cannot handle non-stationary environments, where the best decision is drifting over time.

To address this limitation, we investigate the delayed OCO with a more suitable performance metric called dynamic regret \citep{Zinkevich2003}: \[\Dreg(\u_1,\dots,\u_T)=\sum_{t=1}^Tf_t(\x_t)-\sum_{t=1}^Tf_t(\u_t)\] which compares the player against any sequence of changing comparators $\u_1,\dots,\u_T\in\K$. Let $P_T=\sum_{t=2}^T\|\u_t-\u_{t-1}\|_2$ denote the path-length of comparators. In the case with non-delayed full-information feedback, there exists a well-known algorithm called Ader \citep{Zhang18_ader} that enjoys the optimal dynamic regret bound of $O(\sqrt{T(P_T+1)})$. Moreover,  Ader has also been extended into the case with non-delayed bandit feedback, where an $O(\sqrt{n}T^{3/4}\sqrt{P_T+1})$ dynamic regret bound is achieved \citep{Zhao-JMLR22}. Thus, it is natural to ask whether these algorithms and dynamic regret bounds can be generalized to the case with arbitrary delays. This paper provides 
an affirmative answer by presenting the following results
\begin{compactitem}
\item First, we propose an algorithm called \underline{m}ult\underline{i}p\underline{l}e \underline{d}elayed \underline{o}nline \underline{g}radient \underline{d}escent (Mild-OGD) for the case with delayed full-information feedback, which can be viewed as a delayed variant of Ader. Let $\bar{d}=\sum_{t=1}^Td_t/T$ and $d=\max\{d_1,\dots,d_T\}$ denote the average and maximum delay, respectively. We show that its dynamic regret can be automatically bounded by $O(\sqrt{\bar{d}T(P_T+1)})$ under the in-order assumption (i.e., delays do not change the arrival order of feedback) and $O(\sqrt{dT(P_T+1)})$ in the worst case. A matching $\Omega(\sqrt{dT(P_T+1)})$ lower bound is also derived to demonstrate the optimality of Mild-OGD in the worst case.
\item Second, we develop a variant of Mild-OGD, which is called \underline{m}ult\underline{i}p\underline{l}e \underline{d}elayed \underline{b}andit \underline{g}radient \underline{d}escent (Mild-BGD), for the case with only delayed bandit feedback. We prove that Mild-BGD can achieve an $O((\sqrt{n}T^{3/4}+(n\bar{d})^{1/3}T^{2/3})\sqrt{P_T+1})$ dynamic regret bound under the in-order assumption, and an $O((\sqrt{n}T^{3/4}+\sqrt{dT})\sqrt{P_T+1})$ dynamic regret bound in the worst case. Interestingly, these bounds can even match the dynamic regret bound of existing non-delayed bandit algorithms under a relatively large amount of delay, i.e., $\bar{d}=O(\sqrt{n}T^{1/4})$ and $d=O(n\sqrt{T})$ respectively.
\end{compactitem}

Note that the success of Ader \citep{Zhang18_ader} in minimizing dynamic regret mainly owes to a two-level framework---running multiple instances of online gradient descent (OGD) \citep{Zinkevich2003} as experts and tracking the best one via a meta-algorithm. Thus, our Mild-OGD and its bandit variant also adopt the two-level framework, but do face some new challenges when choosing the expert-algorithm and meta-algorithm for handling arbitrary delays. First, although there exists a delayed variant of OGD \citep{Quanrud15}, its dynamic regret under arbitrary delays is still unclear. Moreover, its key idea is to perform a gradient descent step by using the sum of gradients received in each round, which ignores the arrival order of gradients. As a result, it cannot benefit from the in-order property, which is critical for reducing the dynamic regret according to our novel analysis. To this end, we propose a new delayed variant of OGD to be the expert-algorithm of our Mild-OGD, which performs a gradient descent step for each delayed gradient according to their arrival order. Then, since the performance of the expert-algorithm will also be arbitrarily delayed, the meta-algorithm of our Mild-OGD is built upon delayed Hedge---a technique for prediction with delayed expert advice \citep{DelayHedge}, instead of simply using that of Ader.

Furthermore, compared with Mild-OGD, its bandit variant needs to exploit the delayed feedback more carefully. Inspired by \citet{Zhao-JMLR22}, a natural idea is to simply replace the gradient utilized in Mild-OGD with an approximate one generated via the classical one-point gradient estimator \citep{OBO05}. However, one can show that this approach only achieves a dynamic regret bound of $O((\sqrt{n}T^{3/4}+(nd)^{1/3}T^{2/3})\sqrt{P_T+1})$ in the worst case, whose delay-dependent part is looser than the $O(\sqrt{dT(P_T+1)})$ term in the above desired bound. This gap can be attributed to the joint effect of the delays and the large variance of approximate gradients. To close the gap, our Mild-BGD not only utilizes the one-point gradient estimator, but also adopts a blocking update mechanism that can decouple this joint effect. It is worth noting that although \citet{Wan-NIPS24} have exploited the delayed bandit feedback in a very similar way, they focus on minimizing the static regret. Moreover, for ease of using the delayed information, their algorithm is based on follow-the-regularized-leader (FTRL) \citep{Online:suvery,Hazan2016}, which does not enjoy any guarantees on the dynamic regret. In contrast, the expert-algorithm of our Mild-BGD should still be an OGD-type algorithm, which is more complicated and requires a new analysis.

The rest of the paper is structured as follows. In Section 2, we briefly review related work. Section 3 presents our Mild-OGD and Mild-BGD, along with the corresponding theoretical guarantees. The detailed proofs for these results are provided in Section 4, and we conclude this paper in Section 5. The appendix provides some omitted discussions about the practical implementation and theoretical guarantees of our algorithms. Note that a preliminary version of this paper, which only contains the results in the full-information setting, was presented at the 41st International Conference on Machine Learning in 2024 \citep{Wan-ICML-24}. In this paper, we have significantly enriched the preliminary version by adding the extension into the bandit setting.

\section{Related Work}
In this section, we provide a brief review of related work on OCO with arbitrary delays and dynamic regret, spanning the full-information setting to the bandit setting.
\subsection{Arbitrary Delays}
To deal with arbitrary delays, \citet{Joulani13} first propose a black-box technique called  \underline{b}lack-box \underline{o}nline \underline{l}earning under \underline{d}elayed feedback (BOLD), which can extend any non-delayed OCO algorithm into the delayed setting. The main idea is to pool multiple instances of the non-delayed algorithm, each of which runs over a subsequence of rounds that satisfies the non-delayed assumption. As proved by \citet{Joulani13}, if the non-delayed algorithm has a static regret bound of $\reg(T)$, BOLD can attain a static regret bound of $d \reg(T/d)$. Note that in the non-delayed full-information setting, there exist plenty of algorithms with an $O(\sqrt{T})$ static regret bound, such as OGD \citep{Zinkevich2003}. As a result, combining with them, BOLD can achieve a static regret bound of $O(\sqrt{dT})$, which matches an existing lower bound of $\Omega(\sqrt{dT})$ in the worst case \citep{Weinberger02_TIT}. However, despite its generality, BOLD needs to run multiple instances of the non-delayed algorithm, which could be prohibitively resource-intensive \citep{Quanrud15,Joulani16}. Moreover, in the non-delayed bandit setting, the static regret of existing practical algorithms is worse than $O(\sqrt{T})$, e.g., the $O(\sqrt{n}T^{3/4})$ static regret bound of bandit gradient descent (BGD)---a variant of OGD based on the one-point gradient estimator \citep{OBO05}. A simple combination of BOLD and these bandit algorithms has been proven not to be the optimal solution for handling arbitrary delays \citep{ICML20_Mertikopoulos,Wan-NIPS24}. For these reasons, instead of adopting BOLD, subsequent studies extend many specific non-delayed OCO algorithms into the delayed setting by only running a single instance of them with the delayed information about all loss functions.

Specifically, in the full-information setting, \citet{Quanrud15} propose a delayed variant of OGD, and reduce the static regret to $O(\sqrt{\bar{d}T})$, which depends on the average delay $\bar{d}$, instead of the maximum delay $d$. By additionally introducing the in-order assumption, \citet{NIPS2014_5cce8ded} develop a delayed variant of the adaptive gradient (AdaGrad) algorithm \citep{McMahan10,DADO2011}, and establish a data-dependent static regret bound, which could be tighter than $O(\sqrt{dT})$ for sparse data. Later, \citet{Joulani16} propose another delayed variant of AdaGrad, which can attain a data-dependent static regret bound even without the in-order assumption. Recently, \citet{ICML21_delay} develop delayed variants of optimistic algorithms \citep{Rakhlin13,Joulani17}, which can make use of ``hints" about expected future loss functions to improve the $O(\sqrt{dT})$ static regret. The curvature properties of loss functions have been exploited to reduce the sublinear static regret bound to a logarithmic one \citep{DOGD-SC,Qiu-ICML25}. Moreover, \citet{NeurIPS22-Wan,DOFW-Journal} develop a delayed variant of online Frank-Wolfe \citep{Hazan2012}, and obtain a static regret bound of $O(T^{3/4}+\bar{d}T^{1/4})$. Their algorithm is projection-free and can be efficiently implemented over complex constraints. We also notice that \citet{DelayHedge} consider the problem of prediction with expert advice---a special case of OCO with linear functions and simplex decision sets, and propose a delayed variant of Hedge \citep{Hedge97} to achieve the $O(\sqrt{\bar{d}T})$ static regret.

Different from the full-information setting, OCO with the delayed bandit feedback has rarely been investigated. \citet{ICML20_Mertikopoulos} propose the first algorithm called
\underline{g}radient-free \underline{o}nline \underline{l}earning with \underline{d}elayed feedback (GOLD), which utilizes the oldest received but not utilized loss value to perform an update similar to BGD at each round. Notably, GOLD can achieve a static regret bound of $O(\sqrt{n}T^{3/4}+(nd)^{1/3}T^{2/3})$, which matches the $O(\sqrt{n}T^{3/4})$ static regret of BGD in the non-delayed setting for $d=O(\sqrt{n}T^{1/4})$. In contrast, the combination of BOLD and BGD can only achieve a much worse static regret bound of $O(\sqrt{n}d^{1/4}T^{3/4})$. Later, by utilizing all received but not utilized loss values one by one at each round, \citet{Bistritz-JMLR22} develop an improved variant of GOLD with a static regret bound of $O(\sqrt{n}T^{3/4}+(n\bar{d})^{1/3}T^{2/3})$.\footnote{One may notice that \citet{Bistritz-JMLR22} only argue a static regret bound of $O(nT^{3/4}+\sqrt{n}\bar{d}^{1/3}T^{2/3})$. However, as discussed by \citet{Wan-NIPS24}, it is easy to derive this refined bound by tuning parameters more carefully.} However, the delay-dependent part in this bound still cannot match an $\Omega(\sqrt{\bar{d}T})$ lower bound derived by \citet{Bistritz-JMLR22}. To fill this gap in the worst case, \citet{Wan-NIPS24} propose a novel algorithm called delayed follow-the-bandit-leader (D-FTBL), which enjoys an improved static regret bound of $O(\sqrt{n}T^{3/4}+\sqrt{dT})$. The key insight for their improvement is to further incorporate the delayed bandit feedback with a blocking update mechanism such that the joint effect of the delays and the one-point gradient estimator can be decoupled. Moreover, different from GOLD and its improved variant, which belong to OGD-type algorithms, D-FTBL is based on FTRL \citep{Online:suvery,Hazan2016} for ease of using the delayed information.

\subsection{Dynamic Regret}
Dynamic regret of OCO is first introduced by \citet{Zinkevich2003}, who demonstrates that OGD can attain a dynamic regret bound of $O(\sqrt{T}(P_T+1))$ by simply utilizing a constant learning rate. Later, \citet{Zhang18_ader} establish a lower bound of $\Omega(\sqrt{T(P_T+1)})$ for the dynamic regret. Moreover, to improve the upper bound, \citet{Zhang18_ader} propose a novel algorithm, namely Ader, which runs multiple instances of OGD with different learning rates in parallel, and tracks the best one via Hedge \citep{Hedge97}. Although the strategy of maintaining multiple learning rates is originally proposed to adaptively minimize the static regret for multiple types of functions \citep{Erven16,MetaGrad-JMLR}, \citet{Zhang18_ader} extend it to achieve an optimal dynamic regret bound of $O(\sqrt{T(P_T+1)})$. After that, many efforts have been devoted to deriving guarantees on the dynamic regret under different scenarios \citep{pmlr-v119-cutkosky20a,NIPS20-Zhao,Zhao-JMLR22,Zhao-NIPS-22,BabyCOLT,Baby_AISTATS_22,Baby_AISTATS23,Wang-AAAI-24}. The closest one to this paper is the work of \citet{Zhao-JMLR22}, which for the first time proposes to minimize the dynamic regret in the bandit setting. By combining Ader with the one-point gradient estimator \citep{OBO05}, they establish an $O(\sqrt{n}T^{3/4}\sqrt{P_T+1})$ dynamic regret bound, which can match the $O(\sqrt{n}T^{3/4})$ static regret bound of BGD \citep{OBO05} by setting $P_T=0$. Besides these studies, there also exist plenty of studies \citep{Jadbabaie15,Besbes15,Yang16,Mokhtari16,Zhang17,Zhang18,baby19,Wan-AAAI-2021-B,Wan-COLT23,L4DC:2021:Zhao,Wang_INFOCOM_21,Wang-TON-23,Liu-BCO-25} that focus on a restricted form of the dynamic regret, in which $\u_t=\x_t^\ast\in\argmin_{\x\in\K}f_t(\x)$. However, as discussed by \citet{Zhang18_ader}, the restricted dynamic regret is too pessimistic and less flexible than the general one.

\subsection{Additional Discussions}
Although both arbitrary delays and dynamic regret have attracted much interest, there still lacks of understanding about the effect of arbitrary delays on dynamic regret. More specifically, one may notice that \citet{Wang_INFOCOM_21,Wang-TON-23} have demonstrated under a fixed and knowable delay $d^\prime$, simply performing OGD with a delayed gradient $\nabla f_{t-d^\prime+1}(\x_{t-d^\prime+1})$ is able to achieve a restricted dynamic regret bound of $O(\sqrt{d^\prime T(P_T^\ast+1)})$ when $P_T^\ast=\sum_{t=2}^T\|\x_t^\ast-\x_{t-1}^\ast\|_2$ is also knowable. However, their algorithm and theoretical results do not apply to the general dynamic regret under arbitrary delays. Moreover, after the preliminary version of this paper \citep{Wan-ICML-24}, the dynamic regret of BOLD \citep{Joulani13} has been analyzed by \citet{Wan-COLT25}. Therefore, one may try to extend existing non-delayed algorithms with dynamic regret bounds into the delayed setting via BOLD. However, the analysis of \citet{Wan-COLT25} is limited to non-delayed algorithms with an  $O(T^{\beta}(P_T+1)^{1-\beta})$ dynamic regret bound for some $\beta\in(0,1)$, which is not satisfied by the bandit variant of Ader \citep{Zhao-JMLR22}. In addition, even the combination of BOLD and Ader will only achieve an $O(\sqrt{dT(P_T+1)})$ dynamic regret bound regardless of the in-order assumption \citep{Wan-COLT25}, which is worse than the $O(\sqrt{\bar{d}T(P_T+1)})$ dynamic regret bound of our Mild-OGD.



\section{Main Results}
In this section, we introduce the necessary assumptions, detailed procedures, and the corresponding theoretical guarantees of our two algorithms for full-information and bandit settings, respectively.
\subsection{Necessary Assumptions}
Following previous studies on OCO \citep{Online:suvery,Hazan2016}, it is common to introduce some assumptions on loss functions and the decision set. Specifically, in the full-information setting, the following two assumptions have been widely utilized.
\begin{assum}
\label{assum1}
All loss functions are $G$-Lipschitz over $\K$, i.e., $|f_t(\x)-f_t(\y)|\leq G\|\x-\y\|_2$ for any $\x,\y\in\K$ and $t\in[T]$.
\end{assum}
\begin{assum}
\label{assum2}
The decision set $\K$ contains the origin $\ze$, and its radius is bounded by $R$, i.e., $\|\x\|_2\leq R$ for any $\x\in\K$.
\end{assum}
To further deal with the more challenging bandit setting, an additional assumption on loss functions and a slightly stronger assumption on the decision set are required.
\begin{assum}
\label{assum-b1}
All loss functions are bounded over $\K$, i.e., $|f_t(\x)|\leq M$ for any $\x\in\K$ and $t\in[T]$. Moreover, all functions are chosen beforehand, i.e., the adversary is oblivious.
\end{assum}
\begin{assum}
\label{assum-b2}
The decision set $\K$ is full-dimensional, and there exist two constants $r,R>0$ such that $r\mathcal{B}^n\subseteq\K\subseteq R\mathcal{B}^n$, where $\mathcal{B}^n$ denotes the unit Euclidean ball centered at the origin in $\mathbb{R}^n$.
\end{assum}
In addition, the previously mentioned in-order assumption can be formalized as follows.
\begin{assum}
\label{assum4}
Delays do not change the arrival order of feedback, i.e., the gradient $\nabla f_i(\x_i)$ (or the loss value $f_i(\x_i)$ in the bandit setting) is received before the gradient $\nabla f_j(\x_j)$ (or the loss value $f_j(\x_j)$ in the bandit setting), for any $1\leq i< j\leq T$.
\end{assum}
\begin{remark}
\emph{
Assumption \ref{assum4} stems from parallel and distributed optimization---a representative application of delayed OCO. Specifically, for parallel optimization with many threads, delay is mainly caused by the computing time of gradients. Thus, as in \citet{NIPS2014_5cce8ded}, it is reasonable to assume that these delays satisfy the in-order assumption, because the gradient computed first is more likely to be obtained first. Even for general parallel and distributed optimization, polynomially growing delays, which imply $d_i\leq d_j$ for $i<j$ and thus satisfy the in-order assumption, have received much attention in recent years \citep{Zhou_delay_ICML18,Ren-Delay-AAAI20,Zhou-MOR-Delay-22}. Moreover, we want to emphasize that Assumption \ref{assum4} is only utilized to achieve the dynamic regret bound depending on the average delay $\bar{d}$, and the case without this assumption is also considered.}
\end{remark}
\subsection{Multiple Delayed Online Gradient Descent (Mild-OGD)}
As previously discussed, our Mild-OGD adopts a two-level framework that consists of an expert-algorithm and a meta-algorithm. According to previous studies \citep{Zhang18_ader,NIPS20-Zhao,Zhao-JMLR22}, the expert-algorithm should be able to minimize the dynamic regret for comparators with a specific path-length, and the meta-algorithm should have the power of tracking the best one of multiple instances of the expert-algorithm for different path-lengths. However, they only provide suitable expert-algorithms and meta-algorithms for the non-delayed setting. In the following, we for the first time consider the design of expert-algorithm and meta-algorithm under arbitrary delays. 

\subsubsection{Mild-OGD: Expert-algorithm}
Recall that in the non-delayed setting, the expert-algorithm of Ader \citep{Zhang18_ader} is the classical OGD algorithm \citep{Zinkevich2003}, which updates the decision as
\begin{equation}
\label{OGD}
\begin{split}
\x_{t+1}^\eta=\argmin_{\x\in\K}\|\x-(\x_{t}^\eta-\eta \g_t^\eta)\|_2^2
\end{split}
\end{equation}
at each round $t$, where $\eta$ is a learning rate and $\g_t^\eta$ could be the gradient $\nabla f_t(\x_t^\eta)$ of the expert-algorithm or the gradient of the meta-algorithm.\footnote{The case with $\g_t^\eta=\nabla f_t(\x_t^\eta)$ can be viewed as directly running the expert-algorithm over the loss functions selected by the adversary. In contrast, the latter case is equal to running the expert-algorithm over some surrogate functions constructed by the meta-algorithm, which will be clearly introduced later.}~To handle the setting with arbitrary delays, \citet{Quanrud15} have proposed a delayed variant of OGD by replacing the gradient $\g_t^\eta=\nabla f_t(\x_t^\eta)$ with the sum of gradients received in round $t$. However, it ignores the arrival order of gradients, and thus cannot benefit from the in-order property when minimizing the dynamic regret. To address this limitation, we propose a new delayed variant of OGD to be our expert-algorithm, which performs a gradient descent step for each delayed gradient according to their arrival order. 
\begin{algorithm}[t]
\caption{Mild-OGD: Expert-algorithm}
\label{Ader-Expert}
\begin{algorithmic}[1]
\STATE \textbf{Input:} a learning rate $\eta$
\STATE \textbf{Initialization:} set  $\y_1^\eta=\ze$ and $\tau=1$
\FOR{$t=1,\dots,T$}
\STATE Submit $\x_t^\eta=\y_{\tau}^\eta$ to the meta-algorithm, and query $\g_t^\eta$ from the adversary or meta-algorithm
\STATE Receive gradients $\{\g_t^\eta|k\in\F_t\}$ from the adversary or meta-algorithm
\FOR{$k\in\F_t$ (in the ascending order)}
\STATE Compute $\y_{\tau+1}^\eta=\argmin_{\x\in\K}\|\x-(\y_{\tau}^\eta-\eta\g_k^\eta)\|_2^2$ and set $\tau=\tau+1$
\ENDFOR
\ENDFOR
\end{algorithmic}
\end{algorithm}

The detailed procedure is outlined in Algorithm \ref{Ader-Expert}, where $\tau$ records the number of generated decisions and $\y_{\tau}^\eta$ denotes the $\tau$-th generated decision. Initially, we simply set $\y_1^\eta=\ze$ and $\tau=1$. Then, at each round $t\in[T]$, this algorithm uses the latest decision, i.e., setting $\x_t^\eta=\y_\tau^\eta$, and submits it to the meta-algorithm. After that, this algorithm queries a gradient $\g_t^\eta$ from the adversary or meta-algorithm. Due to the effect of arbitrary delays, we only receive a set of delayed gradients $\{\g_k^\eta|k\in\F_t\}$, where $\F_t=\{k\in[T]|k+d_k-1=t\}$. For each $k\in\F_t$, inspired by (\ref{OGD}), we perform the following update
\begin{equation}
\label{Mild-OGD-Expert-1}
\y_{\tau+1}^\eta=\argmin_{\x\in\K}\|\x-(\y_{\tau}^\eta-\eta\g_k^\eta)\|_2^2
\end{equation}
and then set $\tau=\tau+1$. Moreover, to utilize the in-order property, elements in the set $\F_t$ are sorted and traversed in the ascending order.
\begin{remark}
\label{remark-alg1:label}
\emph{
Note that some gradients may arrive after round $T$. Although our Algorithm \ref{Ader-Expert} does not need to utilize these gradients, they are useful to facilitate our analysis and discussion. Therefore, in the analysis, we virtually set $\x_t^\eta=\y_{\tau}^\eta$ and perform steps $5$ to $8$ in Algorithm \ref{Ader-Expert} at some additional rounds $t=T+1,\dots,T+d-1$. In this way, all the queried gradients are fully exploited to generate decisions $\y_1^\eta,\dots,\y_{T+1}^\eta$. Moreover, we denote the time-stamp of the $\tau$-th utilized gradient by $c_\tau$. One can imagine that Algorithm \ref{Ader-Expert} also sets $c_\tau=k$ at the beginning of its step $7$.
}
\end{remark}

Then, as shown in the following theorem, we first provide an upper bound for the dynamic regret of Algorithm \ref{Ader-Expert} with $\g_t^\eta=\nabla f_t(\x_t^\eta)$.
\begin{thm}
\label{thm1}
Under Assumptions \ref{assum1} and \ref{assum2}, for any comparator sequence $\u_1,\dots,\u_T\in\K$, Algorithm \ref{Ader-Expert} with $\g_t^\eta=\nabla f_t(\x_t^\eta)$ ensures
\begin{equation}
\label{eq1-thm1-main}
\begin{split}
\sum_{t=1}^Tf_t(\x_t^\eta)-\sum_{t=1}^Tf_t(\u_t)
\leq \frac{R^2+RP_T}{\eta}+{\eta G^2}\sum_{t=1}^T(m_t+1)+\sum_{t=1}^TG\|\u_t-\u_{c_t}\|_2
\end{split}
\end{equation}
where $m_t=t-1-\sum_{i=1}^{t-1}|\F_i|$.
\end{thm}
The value of $m_t$ actually counts the number of gradients that have been queried, but still not received at the end of round $t-1$. Since the gradient $\g_t^\eta$ will only be counted as an unreceived gradient in $d_t-1$ rounds, it is easy to verify that
\begin{equation}
\label{thm01-eq6pre}
\sum_{t=1}^T(m_t+1)\leq\sum_{t=1}^Td_t=\bar{d}T.
\end{equation}
Therefore, the first two terms in the right side of (\ref{eq1-thm1-main}) are bounded by $O(\sqrt{\bar{d}T(P_T+1)})$ so long as 
\begin{equation}
\label{eta-dogd}
\eta\propto\sqrt{\frac{P_T+1}{\sum_{t=1}^T(m_t+1)}}.
\end{equation}
However, we still need to bound the last term in the right side of (\ref{eq1-thm1-main}), which reflects the ``comparator drift'' caused by arbitrary delays, and has never appeared in previous studies on the delayed feedback and dynamic regret. To this end, we establish the following lemma.
\begin{lem}
\label{lem1_DOGD_main}
Under Assumption \ref{assum2}, for any comparator sequence $\u_1,\dots,\u_T\in\K$, Algorithm \ref{Ader-Expert} ensures
\begin{equation*}
\sum_{t=1}^T\|\u_t-\u_{c_t}\|_2\leq\min\left\{2RK,2dP_T\right\}\leq2\sqrt{dKRP_T}
\end{equation*}
where $K=\sum_{t=1}^T\mathbb{I}(t\neq c_t)$ and $\mathbb{I}(\cdot)$ denotes the indicator function.
\end{lem}
Since Algorithm \ref{Ader-Expert} utilizes the received gradients in the ascending order, the value of $K$ counts the number of delays that are not in order. Therefore, Lemma \ref{lem1_DOGD_main} implies that the comparator drift can be upper bounded by $O(\sqrt{dTP_T})$ in the worst case because of $K\leq T$, and vanishes if the in-order property holds, i.e., $K=0$. To facilitate discussions, we focus on these two extremes, though the comparator drift can be bounded by $O(\sqrt{\bar{d}TP_T})$ in an intermediate case with $K\leq O(T\bar{d}/d)$.
\begin{remark}
\emph{From the above discussions, it seems that Algorithm \ref{Ader-Expert} with $\eta$ in \eqref{eta-dogd} is sufficient to automatically achieve the $O(\sqrt{\bar{d}T(P_T+1)})$ dynamic regret bound under the in-order assumption and the $O(\sqrt{dT(P_T+1)})$ dynamic regret bound in the worst case. Unfortunately, such results actually do no hold for any sequence of comparators, because \eqref{eta-dogd} depends on a specific path-length. To tackle this issue, a naive attempt is to remove the dependence of $\eta$ on the path-length. Then, Algorithm \ref{Ader-Expert} can automatically achieve an $O(\sqrt{\bar{d}T}(P_T+1))$ dynamic regret bound under the in-order assumption and the $O(\sqrt{dT}(P_T+1))$ dynamic regret bound in the worst case. However, even in the non-delayed setting, these bounds can only recover the $O(\sqrt{T}(P_T+1))$ dynamic regret of OGD \citep{Zinkevich2003}, rather than the $O(\sqrt{T(P_T+1)})$ dynamic regret of Ader \citep{Zhang18_ader}. To reduce the dependence on $P_T$, we still require the following meta-algorithm.}
\end{remark}

\subsubsection{Mild-OGD: Meta-algorithm}
Following Ader \citep{Zhang18_ader}, the meta-algorithm should maintain multiple experts, each of which is an instance of the expert-algorithm with a different learning rate. Therefore, our meta-algorithm first activates a set of experts $\left\{E^{\eta}|\eta\in\H\right\}$ by invoking Algorithm \ref{Ader-Expert} for each $\eta\in\H$, where $\H$ denotes the set of learning rates for experts. Then, in each round $t\in[T]$, each expert $E^{\eta}$ will generate a decision $\x_t^\eta$, and the meta-algorithm plays a weighted decision $\x_t=\sum_{\eta\in\H}w_{t}^\eta\x_t^\eta$, where $w_{t}^\eta\in[0,1]$ is the weight of the expert $E^{\eta}$. Note that there exist two options for the meta-algorithm to maintain these experts:~running them over the original functions $\{f_t(\x)\}_{t\in[T]}$ or the surrogate functions $\{\ell_t(\x)\}_{t\in[T]}$, where
$\ell_t(\x)=\langle\nabla f_t(\x_t),\x\rangle$.\footnote{One may notice that this definition is not exactly the same as $\ell_t(\x)=\langle\nabla f_t(\x_t),\x-\x_t\rangle$ used in Ader \citep{Zhang18_ader}. However, it is easy to verify that both of them play the same role.}~Here, we choose the second option, because surrogate functions allow experts to reuse the gradient of the meta-algorithm and thus can avoid inconsistent delays between the meta-algorithm and experts. Specifically, in each round $t\in[T]$, only the meta-algorithm queries the gradient $\nabla f_t(\x_t)$ from the adversary, and receives gradients $\{\nabla f_k(\x_k)|k\in\F_t\}$ due to the effect of arbitrary delays. Then, the gradients $\{\nabla f_k(\x_k)|k\in\F_t\}$ will be sent to each expert $E^\eta$ so that they can update their own decisions without querying additional gradients.\footnote{In other words, when serving as an expert of the meta-algorithm, Algorithm \ref{Ader-Expert} only requires $\g_t^\eta=\nabla f_t(\x_t)$.}
\begin{algorithm}[t]
\caption{Mild-OGD: Meta-algorithm}
\label{Ader-Meta}
\begin{algorithmic}[1]
\STATE \textbf{Input:} a parameter $\alpha$ and a set $\H$ containing learning rates for experts
\STATE Activate experts $\left\{E^{\eta}|\eta\in\H\right\}$ by invoking Algorithm \ref{Ader-Expert} for each learning rate $\eta\in\H$
\STATE Sort learning rates as $\eta_1\leq\dots\leq\eta_{|\H|}$, and set $w_1^{\eta_i}=\frac{|\H|+1}{i(i+1)|\H|},\forall i\in[|\H|]$
\FOR{$t=1,\dots,T$}
\STATE Receive $\x_t^\eta$ from each expert $E^\eta$, and play the decision $\x_t=\sum_{\eta\in\H}w_{t}^\eta\x_t^\eta$
\STATE Query $\nabla f_t(\x_t)$, receive $\{\nabla f_k(\x_k)|k\in\F_t\}$, and send $\{\nabla f_k(\x_k)|k\in\F_t\}$ to each expert $E^\eta$
\STATE Update the weight of each expert $E^{\eta}$ as in \eqref{delayH-24}
\ENDFOR
\end{algorithmic}
\end{algorithm}

Finally, we notice that in the non-delayed setting, Ader  adopts initial weights that depend on the sorting of these learning rates, and then updates them based on the historical performance of these experts. Let $\eta_i$ denote the $i$-th smallest learning rate in $\H$. It is natural to reuse the initialization strategy, i.e., setting $w_1^{\eta_i}=(|\H|+1)/(i(i+1)|\H|),\forall i\in[|\H|]$. However, the method used by Ader for updating the weights is the vanilla Hedge \citep{LEA97}, which requires the gradient $\nabla f_t(\x_t)$ at each round $t\in[T]$, and thus is incompatible to the delayed setting studied
here. To address this limitation, we utilize the delayed Hedge \citep{DelayHedge}, an expert-tracking method under arbitrary delays, to update the weight of each expert as
\begin{equation}
 \label{delayH-24}
w_{t+1}^\eta=\frac{w_t^\eta e^{-\alpha\sum_{k\in\F_t}\ell_{k}(\x_k^\eta)}}{\sum_{\mu\in\H}w_t^{\mu}e^{-\alpha\sum_{k\in\F_t}\ell_{k}(\x_k^\mu)}}
\end{equation}
where $\alpha$ is a parameter and $\ell_k(\x)=\langle\nabla f_k(\x_k),\x\rangle$.~This is the critical difference between our meta-algorithm and that of Ader. 

The detailed procedure of our meta-algorithm is summarized in Algorithm \ref{Ader-Meta}, and we have the following theoretical guarantee for the dynamic regret of Mild-OGD.
\begin{thm}
\label{thm2}
Let $m_t=t-1-\sum_{i=1}^{t-1}|\F_i|$. Under Assumptions \ref{assum1} and \ref{assum2}, by setting
\[\H=\left\{\eta_i=\left.\frac{2^{i}R}{G\sqrt{2\beta}}\right|i=1,\dots,N\right\}\text{ and }\alpha=\frac{1}{GR\sqrt{\beta}}\]
where $N=\lceil(1/2)\log_2 T\rceil+1$ and $\beta=\sum_{t=1}^T(m_t+1)$, Algorithm \ref{Ader-Meta} ensures
\begin{align*}
\sum_{t=1}^Tf_t(\x_t)-\sum_{t=1}^Tf_t(\u_t)
\leq (3\sqrt{R(2R+P_T)}+2R\ln(k+1)+R)G\sqrt{\bar{d}T}+C
\end{align*}
for any comparator sequence $\u_1,\dots,\u_T\in\K$, where $k=\lfloor\log_2\sqrt{(P_T+2R)/2R}\rfloor+1$ and 
\begin{equation}
\label{bigC}
C=\left\{
\begin{aligned}
&0,~\text{if Assumption \ref{assum4} also holds;}\\
&\min\left\{2TGR,2dGP_T\right\},~\text{otherwise.}
\end{aligned}
\right.
\end{equation}
\end{thm}
From Theorem \ref{thm2}, Mild-OGD attains an $O(\sqrt{\bar{d}T(P_T+1)}+C)$ dynamic regret bound, where $C$ is an adaptive bound derived from Lemma \ref{lem1_DOGD_main} for the comparator drift. Due to $C\leq 2G\sqrt{dTRP_T}$, this dynamic regret bound becomes $O(\sqrt{dT(P_T+1)})$ in the worst case and $O(\sqrt{\bar{d}T(P_T+1)})$ under the in-order assumption automatically.  
Compared with the dynamic regret of simply running Algorithm \ref{Ader-Expert} over the original losses, Mild-OGD reduces the linear dependence on $P_T$ to be sublinear, and now only magnifies the optimal $O(\sqrt{T(P_T+1)})$ bound achieved in the non-delayed setting \citep{Zhang18_ader} by a factor of $\sqrt{\bar{d}}$ or $\sqrt{d}$. Moreover, for $P_T=0$, Mild-OGD can  recover the $O(\sqrt{\bar{d}T})$ static regret bound \citep{Quanrud15} even without the in-order assumption.
\begin{remark}
\label{rem4}
\emph{
One possible concern about Mild-OGD is that Theorem \ref{thm2} needs to tune parameters based on the value of $\sum_{t=1}^T(m_t+1)$, which is generally unknown beforehand. However, we want to emphasize that \citet{Quanrud15} also face this issue when minimizing the static regret of OCO with arbitrary delays, and have introduced a simple solution by utilizing the standard ``doubling trick'' \citep{LEA97} to adaptively adjust the learning rate. The main insight behind this solution is that the value of $\sum_{t=1}^T(m_t+1)$ can be calculated on the fly. The details about Mild-OGD with the doubling trick are provided in the appendix.
}
\end{remark}

Besides, inspired by the existing lower bound in the non-delayed setting \citep{Zhang18_ader}, we provide the
first lower bound for the dynamic regret of OCO with arbitrary delays.
\begin{thm}
\label{thm-low}
Let $L=\left\lceil 2TR/\max\{P,2R\}\right\rceil$. Suppose $\K=\left[-R/\sqrt{n},R/\sqrt{n}\right]^n$ which satisfies Assumption \ref{assum2}. For any OCO algorithm, any $P\in[0,2TR]$, and any positive integer $d$, there exists a sequence of comparators $\u_1,\dots,\u_T\in\K$ satisfying $P_T\leq P$, a sequence of functions $f_1(\x),\dots,f_T(\x)$ satisfying Assumption \ref{assum1}, and a sequence of delays $1\leq d_1,\dots,d_T\leq d$
such that
\begin{equation*}
\sum_{t=1}^Tf_t(\x_t)-\sum_{t=1}^Tf_t(\u_t)\geq\left\{
\begin{aligned}
&\frac{RGT}{\sqrt{2}},~\text{if $d> L$;}\\
&\frac{G\sqrt{dR\max\{P,2R\}T}}{4},~\text{otherwise.}
\end{aligned}
\right.
\end{equation*}
\end{thm}
From Theorem \ref{thm-low}, if $d>L=\Omega(T/(P_T+1))$, there exists an $\Omega(T)$ lower bound on the dynamic regret, which can be trivially matched by any OCO algorithm including our Algorithm \ref{Ader-Meta}. Therefore, we mainly focus on the case $d\leq L$, and notice that Theorem \ref{thm-low} essentially establishes an $\Omega(\sqrt{dT(P_T+1)})$ lower bound, which matches the $O(\sqrt{dT(P_T+1)})$ dynamic regret of our Mild-OGD in the worst case.

\subsection{Multiple Delayed Bandit Gradient Descent (Mild-BGD)}
Now, we proceed to deal with the delayed bandit setting, where only the loss value of the played decision can be queried at each round and will be received after an arbitrary delay. Compared with the delayed full-information setting, two new challenges prevent the application of the previously mentioned Mild-OGD. First, it requires the gradient of the played decision, which is no longer available. To address this challenge, a natural solution is to approximate the required gradient by applying the classical one-point gradient estimator \citep{OBO05} to the bandit feedback. However, the second challenge arises when using the approximate gradient. Specifically, the approximate gradient actually has a large variance, and a straightforward application of it will amplify the original effect of the delays. In the following, we provide necessary preliminaries about the delayed bandit setting, and develop the bandit variant of Mild-OGD, namely Mild-BGD, which can address these two challenges.

\subsubsection{Preliminaries and Warm-up without the Two-level Framework}
Let $\mathcal{B}^n$ and $\mathcal{S}^n$ denote the unit Euclidean ball and the unit Euclidean sphere centered at the origin in $\mathbb{R}^n$, respectively. The one-point gradient estimator can be formalized as the following lemma.
\begin{lem}
\label{estimator-lem1}
(Lemma 1 in \citet{OBO05})  For a function $f(\x) : \mathbb{R}^n \mapsto \mathbb{R}$ and a constant $\delta >0$, its $\delta$-smoothed version defined as $\hat{f}_\delta(\mathbf{x})=\mathbb{E}_{\mathbf{s}\sim\mathcal{B}^n}\left[f(\mathbf{x}+\delta\mathbf{s})\right]$ ensures
\[
\nabla \hat{f}_{\delta}(\mathbf{x}) = \mathbb{E}_{\mathbf{s} \sim \mathcal{S}^n} \left[ \frac{n}{\delta} f(\mathbf{x} + \delta \mathbf{s}) \mathbf{s} \right].
\]
\end{lem}
From Lemma \ref{estimator-lem1}, the randomized vector $\frac{n}{\delta} f(\mathbf{x} + \delta \mathbf{s}) \s$, which can be computed by only using a single value of the function $f(\x)$, is an unbiased gradient estimator of its $\delta$-smoothed version $\hat{f}_\delta(\mathbf{x})$. Moreover, it is easy to verify that $f(\x)$ is close to $\hat{f}_\delta(\mathbf{x})$, and thus this randomized vector is also sufficient to serve as the one-point gradient estimator of $f(\x)$.

To help understand the above second challenge, we provide a simple combination of Algorithm \ref{Ader-Expert} with the one-point gradient estimator before introducing our Mild-BGD. Specifically, to exploit this estimator, under Assumption \ref{assum-b2}, we define a shrunk set $\K_\delta=(1-\delta/r)\K=\{(1-\delta/r)|\x\in\K\}$, which ensures that $\x+\delta\mathbf{s}\in\K$ for any $\x\in\K_\delta$, $\mathbf{s} \sim \mathcal{S}^n$, and $\delta\in(0,r)$. The simple combination of Algorithm \ref{Ader-Expert} with this estimator is to replace $\K$ in step $7$ with $\K_\delta$,  and query $\g_t^\eta=\frac{n}{\delta} f_t(\hat{\x}_t^\eta) \mathbf{s}_t$ in step $4$, where $\hat{\x}_t^\eta={\x}_t^\eta+\delta \mathbf{s}_t$ and $\mathbf{s}_t \sim \mathcal{S}^n$. Then, combining the analysis of Algorithm \ref{Ader-Expert} with properties of the one-point gradient estimator, we have the following theoretical guarantee.
\begin{thm}
\label{thm1-bandit}
Let $m_t=t-1-\sum_{i=1}^{t-1}|\F_i|$. Under Assumptions \ref{assum1}, \ref{assum-b1}, and \ref{assum-b2}, for any comparator sequence $\u_1,\dots,\u_T\in\K$, Algorithm \ref{Ader-Expert} with $\K=\K_\delta$ and $\g_t^\eta=\frac{n}{\delta} f_t(\hat{\x}_t^\eta) \mathbf{s}_t$, where $\hat{\x}_t^\eta=\x_t^\eta+\delta \mathbf{s}_t$, $\mathbf{s}_t \sim \mathcal{S}^n$, and $\delta\in(0,r)$, ensures
\begin{equation*}
\begin{split}
\mathbb{E}\left[\sum_{t=1}^Tf_t(\hat{\x}_t^\eta)-\sum_{t=1}^Tf_t(\u_t)\right]
\leq& \frac{R^2+RP_T}{\eta}+\frac{\eta Tn^2M^2}{2\delta^2}+\frac{\eta nGM}{\delta}\sum_{t=1}^Tm_t\\
& +\sum_{t=1}^TG\left\|{\u}_t-{\u}_{c_t}\right\|_2+3\delta GT+\frac{\delta GRT}{r}.
\end{split}
\end{equation*}
\end{thm}
Compared with the upper bound in Theorem \ref{thm1}, there are two changes: i) the terms proportional to $\eta$ now are bounded by $O(\eta T(n^2\delta^{-2}+n\bar{d}\delta^{-1}))$, instead of $O(\eta T\bar{d})$; ii) two additional terms bounded by $O(\delta T)$ are introduced. Note that the former change is due to the variance of estimated gradients, and the latter change is due to the difference between the original function and its $\delta$-smoothed version. Due to the latter change, the parameter $\delta$ must be sufficiently small, which results in a large variance of estimated gradients. Thus, the former change will cause a negative effect to the dynamic regret, which requires an adjustment of the parameter $\eta$. To be precise, by tuning these parameters approximately, we  establish the following corollary.
\begin{cor}
\label{cor1-bandit:label}
Let $\delta=\max\left\{\sqrt{n}/{T^{1/4}},(n\beta)^{1/3}/{T^{2/3}}\right\}$, $\eta=\min\left\{c/(\sqrt{n}T^{3/4}),c/((n\beta)^{1/3}T^{1/3})\right\}$, where $c= R/M$, $\beta=\sum_{t=1}^Tm_t$, and $m_t=t-1-\sum_{i=1}^{t-1}|\F_i|$, and assume that $\delta\in(0,r)$ for brevity.\footnote{It can always be satisfied by multiplying $\delta$ with a proper constant, which does not affect the final bound.}~Under Assumptions \ref{assum1}, \ref{assum-b1}, and \ref{assum-b2}, for any comparator sequence $\u_1,\dots,\u_T\in\K$, Algorithm \ref{Ader-Expert} with $\K=\K_\delta$ and $\g_t^\eta=\frac{n}{\delta} f_t(\hat{\x}_t^\eta) \mathbf{s}_t$, where $\hat{\x}_t^\eta=\x_t^\eta+\delta \mathbf{s}_t$ and $\mathbf{s}_t \sim \mathcal{S}^n$, ensures
\begin{equation}
\label{cor1-mainresult}
   \E\left[\sum_{t=1}^Tf_t(\hat{\x}_t^\eta)-\sum_{t=1}^Tf_t(\u_t)\right]=O\left((\sqrt{n}T^{3/4}+(n\bar{d})^{1/3}T^{2/3})(P_T+1)+C\right)
\end{equation}
where $C$ is defined in \eqref{bigC}.
\end{cor}
Combining \eqref{cor1-mainresult} with \eqref{bigC}, the simple combination of Algorithm \ref{Ader-Expert} with the one-point gradient estimator can achieve an $O((\sqrt{n}T^{3/4}+(nd)^{1/3}T^{2/3})(P_T+1))$ dynamic regret bound in the worst case and an $O((\sqrt{n}T^{3/4}+(n\bar{d})^{1/3}T^{2/3})(P_T+1))$ dynamic regret bound under the in-order assumption automatically. As the common price for handling the bandit feedback, these two bounds are worse than those achieved by Algorithm \ref{Ader-Expert} in the delayed full-information setting. Moreover, for $P_T=0$, this simple combination can recover the existing $O(\sqrt{n}T^{3/4}+(n\bar{d})^{1/3}T^{2/3})$ static regret bound \citep{Bistritz-JMLR22} even without the in-order assumption. However, it still fails to recover the existing $O(\sqrt{n}T^{3/4}+\sqrt{dT})$ static regret bound \citep{Wan-NIPS24}.

To fill this gap, inspired by \citet{Wan-NIPS24}, we further exploit a blocking update mechanism when extending Algorithm \ref{Ader-Expert} into the delayed bandit setting. Specifically, we first equally divide the total $T$ rounds into $Z=T/K$ blocks, where $K$ is the block size and $Z$ is assumed to be an integer without loss of generality. For any round $t$ in each block $z\in[Z]$, i.e., $t\in\{(z-1)K+1,\dots,zK\}$, we will only maintain  a fixed shrunk decision $\x_{z}^\eta=\y_\tau^\eta\in\K_\delta$, and play the decision $\hat{\x}_t^\eta=\x_z^\eta+\delta\s_t$ with $\s_t\sim\mathcal{S}^n$. Then, we will still receive the delayed loss values $\{f_k(\hat{\x}_k^\eta)|k\in\F_t\}$ at each round $t$ and can generate an estimated gradient by applying the one-point gradient estimator to each value. However, we will only update $\y_\tau$ at the end of each block $z$, instead of each round $t$. Moreover, each update of $\y_\tau$ is now performed according to the cumulative estimated gradients of all rounds in a certain block, instead of the estimated gradient of a single round. Thus, for any $k\in\F_t$, we find the corresponding block $z_k=\lfloor k/K\rfloor$ and update the cumulative estimated gradients of this block as $\g_{z_k}^\eta=\g_{z_k}^\eta+\frac{n}{\delta} f_k(\hat{\x}_k^\eta) \mathbf{s}_k$, where $\g_{z_k}^\eta$ is initialized to be $\ze$. If all the rounds in the block $z_k$ have been included, i.e., $\g_{z_k}^\eta=\sum_{i=(z_k-1)K+1}^{z_kK}\frac{n}{\delta} f_i(\hat{\x}_i^\eta) \mathbf{s}_i$, the block index $z_k$ will be added into a set $\mathcal{A}_z$, which is initialized as $\mathcal{A}_z=\emptyset$. Finally, at the end of each block $z$, similar to \eqref{Mild-OGD-Expert-1}, we perform the following update
\[
    \y_{\tau+1}^\eta=\argmin_{\x\in\K_\delta}\|\x-(\y_{\tau}^\eta-\eta\g_k^\eta)\|_2^2
\]
and set $\tau=\tau+1$ for each $k\in\mathcal{A}_z$, where elements in the set $\A_z$ are also sorted and traversed in the ascending order. The detailed procedure is outlined in Algorithm \ref{BDBGD}, and it is named as delayed bandit gradient descent via blocking update. 

\begin{algorithm}[t]
\caption{Delayed Bandit Gradient Descent via Blocking Update}
\label{BDBGD}
\begin{algorithmic}[1]
\STATE \textbf{Input:} a learning rate $\eta$, a parameter $\delta$, and a block size $K$
\STATE \textbf{Initialization:} set $\y_1^\eta=\ze$, $\tau=1$, and $Z=T/K$
\FOR{$z=1,\dots,Z$}
\STATE Set $\g_z^\eta=\ze$, $\x_z^\eta=\y_{\tau}^\eta$, and $\mathcal{A}_z=\emptyset$
\FOR{$t=(z-1)K+1,\dots,zK$}
\STATE Play $\hat{\x}_t^\eta=\x_z^\eta+\delta\s_t$, where $\s_t\sim\mathcal{S}^n$, query $f_t(\hat{\x}_t^\eta)$, and receive $\{f_k(\hat{\x}_k^\eta)|k\in\F_t\}$
\FOR{$k\in\F_t$}
\STATE Set $z_k=\lfloor k/K\rfloor$ and update $\g_{z_k}^\eta=\g_{z_k}^\eta+\frac{n}{\delta} f_k(\hat{\x}_k^\eta) \mathbf{s}_k$
\STATE If $\g_{z_k}^\eta=\sum_{i=(z_k-1)K+1}^{z_kK}\frac{n}{\delta} f_i(\hat{\x}_i^\eta) \mathbf{s}_i$, set $\mathcal{A}_z=\mathcal{A}_z\cup\{z_k\}$
\ENDFOR
\ENDFOR
\FOR{$k\in\mathcal{A}_z$ (in the ascending order)}
\STATE Compute $\y_{\tau+1}^\eta=\argmin_{\x\in\K_\delta}\|\x-(\y_{\tau}^\eta-\eta\g_k^\eta)\|_2^2$ and set $\tau=\tau+1$
\ENDFOR
\ENDFOR
\end{algorithmic}
\end{algorithm}
\begin{remark}
\label{rem5:label}
\emph{Algorithm \ref{BDBGD} with $K=1$ reduces to the simple combination of Algorithm \ref{Ader-Expert} with the one-point gradient estimator, and thus also enjoys Theorem \ref{thm1-bandit} and Corollary \ref{cor1-bandit:label}. Here, we proceed to consider Algorithm \ref{BDBGD} with any possible choice of $K$. Note that as discussed in Remark \ref{remark-alg1:label}, some feedback may arrive after round $T$, which will not be used by our algorithm but are useful to facilitate our analysis and discussion. Thus, in the analysis of Algorithm \ref{BDBGD}, we virtually introduce some additional blocks $Z+1,\dots,\lceil (T+d-1)/K\rceil$ after the final block $z=Z$.  During these additional blocks, we repeat the original processes in Algorithm \ref{BDBGD} expect for playing $\hat{\x}_t^\eta=\x_z^\eta+\delta\s_t$ and querying $f_t(\hat{\x}_t^\eta)$. In this way, all the queried loss values are fully exploited to generate shrunk decisions $\y_1^\eta,\dots,\y_{Z+1}^\eta$. Moreover, we now denote the block index of the $\tau$-th utilized cumulative estimated gradients by $c_\tau$, i.e., setting $c_\tau=k$ at the beginning of step $13$ in Algorithm \ref{BDBGD}.
}
\end{remark}

Combining with the above notations, we establish the following guarantee for Algorithm \ref{BDBGD}.
\begin{thm}
\label{thm2-bandit}
Let $m_z=z-1-\sum_{i=1}^{z-1}|\mathcal{A}_i|$, $q_z=(z-1)K+1$, and $\gamma=K({nM}/{\delta})^2+(K^2-K)G^2$. Under Assumptions \ref{assum1}, \ref{assum-b1}, and \ref{assum-b2}, for any comparator sequence $\u_1,\dots,\u_T\in\K$, Algorithm \ref{BDBGD} with $\delta\in(0,r)$ ensures
\begin{equation*}
\begin{split}
\mathbb{E}\left[\sum_{t=1}^Tf_t(\hat{\x}_t^\eta)-\sum_{t=1}^Tf_t(\u_t)\right]\leq& \frac{R^2+R{P}_T}{\eta}+\frac{\eta\gamma Z}{2}+\sqrt{\gamma}\eta KG\sum_{z=1}^Zm_z+\sum_{z=1}^ZKG\|{\u}_{q_z}-{\u}_{q_{c_z}}\|_2\\
&+3\delta GT+\frac{\delta GRT}{r}+G\sqrt{2(K-1)RTP_T}.
\end{split}
\end{equation*}
\end{thm}
Interestingly, if $K=(n/\delta)^2$, the dominant delay-dependent term in Theorem \ref{thm2-bandit} can be bounded as
\[
    O\left(\sqrt{K\left(\frac{n}{\delta}\right)^2+K^2}\eta K\sum_{z=1}^Zm_z\right)=O\left(\eta K^2\sum_{z=1}^Zm_z\right)=O\left(\eta dT\right)
\]
where the last equality is due to the fact that $\sum_{z=1}^Zm_z$ reflects the sum of block-level delays and is bounded by $O(Zd/K)$ (see Section \ref{sub:proof_of_corollary_ref_cor2_bandit_label} for details). In this way, the effect of delays is no longer joint with that of the variance of estimated gradients, which allows us to achieve the following result. 
\begin{cor}
\label{cor2-bandit:label}
Let $\delta=\sqrt{n}/{T^{1/4}}$, $K=n\sqrt{T}$, $\eta=\min\left\{c/(\sqrt{n}T^{3/4}),c/(n\sqrt{T\beta^\prime})\right\}$, where $c= R/\sqrt{G^2+M^2}$, $\beta^\prime=\sum_{z=1}^Zm_z$, and $m_z=z-1-\sum_{i=1}^{z-1}|\mathcal{A}_i|$, and assume that $\delta\in(0,r)$ for brevity. Under Assumptions \ref{assum1}, \ref{assum-b1}, and \ref{assum-b2}, for any comparator sequence $\u_1,\dots,\u_T\in\K$, Algorithm \ref{BDBGD} ensures
\begin{equation*}
   \E\left[\sum_{t=1}^Tf_t(\hat{\x}_t^\eta)-\sum_{t=1}^Tf_t(\u_t)\right]=O\left((\sqrt{n}T^{3/4}+\sqrt{dT})(P_T+1)\right).
\end{equation*}
\end{cor}
\begin{remark}
\emph{
    From Corollaries \ref{cor1-bandit:label} and \ref{cor2-bandit:label}, our Algorithm \ref{BDBGD} with appropriate parameters can respectively recover the two best existing static regret bounds for the delayed bandit setting \citep{Bistritz-JMLR22,Wan-NIPS24}. However, even if $d=1$, these two corollaries still cannot recover the existing $O(\sqrt{n}T^{3/4}\sqrt{P_T+1})$ dynamic regret bound for the non-delayed bandit setting \citep{Zhao-JMLR22}. To address this limitation, we need to combine Algorithm \ref{BDBGD} with the two-level framework, which results in the following Mild-BGD algorithm.
}
\end{remark}

\subsubsection{Mild-BGD with the Two-level Framework}
While the two-level framework follows  Mild-OGD, there are still some technical details that require careful attention for handling the delayed bandit setting. Specifically, from the above discussions, the meta-algorithm should also employ the blocking update mechanism, i.e., only maintaining one weight $w_z^\eta$ for each expert $E^\eta$ over all the rounds in each block $z$. Moreover, the one-point gradient estimator should be applied to the decision of the meta-algorithm, rather than that of experts. Therefore, instead of directly running Algorithm \ref{BDBGD} over the set $\K$, each expert $E^\eta$ actually only needs to generate a shrunk decision $\x_z^\eta\in\K_{\delta}$ at the end of each block $z$. In this way, the meta-algorithm can play the decision $\hat{\x}_t={\x}_z+\delta\s_t$ and query the loss value $f_t(\hat{\x}_t)$ at each round $t$ in block $z$, where ${\x}_z=\sum_{\eta\in\H}w_z^\eta\x_z^\eta$ and $\s_t\sim \mathcal{S}^n$. Due to the effect of delays, loss values $\{f_k(\hat{\x}_k)|k\in\F_t\}$ will be received at the end of each round $t$. For each $k\in\F_t$, following Algorithm \ref{BDBGD}, we can update the cumulative estimated gradients of block $z_{k}=\lfloor k/K\rfloor$ as $\g_{z_k}=\g_{z_k}+\frac{n}{\delta} f_k(\hat{\x}_k) \mathbf{s}_k$, and add $z_k$ into the set $\mathcal{A}_z$ if $\g_{z_k}=\sum_{i=(z_k-1)K+1}^{z_kK}\frac{n}{\delta} f_i(\hat{\x}_i) \mathbf{s}_i$, where we initialize $\g_{z_k}=\ze$ and $\mathcal{A}_z=\emptyset$. Finally, it is natural to update the shrunk decision $\x_z^\eta$ of each expert $E^\eta$ by following Algorithm \ref{BDBGD} but replacing $\g_k^\eta$ with $\g_k$ for any $k\in\mathcal{A}_z$. For the weight $w_z^\eta$ of each expert $E^\eta$, inspired by \eqref{delayH-24},  we perform the following update
\begin{equation}
 \label{delayH-24-bandit}
w_{z+1}^\eta=\frac{w_z^\eta e^{-\alpha\sum_{k\in\mathcal{A}_z}\ell_{k}(\x_k^\eta)}}{\sum_{\mu\in\H}w_z^{\mu}e^{-\alpha\sum_{k\in\mathcal{A}_z}\ell_{k}(\x_k^\mu)}}
\end{equation}
where $\ell_k(\x)=\langle\g_{k},\x\rangle$ now is a surrogate function over the shrunk set $\K_{\delta}$.

The detailed procedure of Mild-BGD is summarized in Algorithm \ref{Mild-BGD}, and we have the following guarantees.
\begin{thm}
\label{thm3-bandit}
Let $N=\lceil(1/2)\log_2 T \rceil+1$, $c= R/(\sqrt{2(G^2+M^2)}\max\left\{\sqrt{n}T^{3/4},n\sqrt{T\beta^\prime}\right\})$, and $\beta^\prime=\sum_{z=1}^Zm_z$, where $m_z=z-1-\sum_{i=1}^{z-1}|\A_i|$. Under Assumptions \ref{assum1}, \ref{assum-b1}, and \ref{assum-b2}, Algorithm \ref{Mild-BGD} with $\delta=\sqrt{n}/{T^{1/4}}\in(0,r)$, $K=n\sqrt{T}$, $\alpha= \sqrt{2}c/R^2$, and $\H=\{\eta_i=2^{i}c |i=1,\dots,N\}$ 
ensures
\begin{equation}
\label{thm6-for-dis}
\E\left[\sum_{t=1}^Tf_t(\hat{\x}_t)-\sum_{t=1}^Tf_t(\u_t)\right]=O\left((\sqrt{n}T^{3/4}+\sqrt{dT})\sqrt{P_T+1}\right)
\end{equation}
for any comparator sequence $\u_1,\dots,\u_T\in\K$.
\end{thm}
\begin{thm}
\label{thm4-bandit}
Let $N=\lceil(1/2)\log_2 T \rceil+1$, $c= R/(\sqrt{2}M\max\left\{\sqrt{n}T^{3/4},(n\beta T)^{1/3}\right\})$, and $\beta=\sum_{t=1}^Tm_t$, where $m_t=t-1-\sum_{i=1}^{t-1}|\F_i|$. Under Assumptions \ref{assum1}, \ref{assum-b1}, \ref{assum-b2}, and \ref{assum4}, Algorithm \ref{Mild-BGD} with $\delta=\max\left\{\sqrt{n}/{T^{1/4}},(n\beta)^{1/3}/{T^{2/3}}\right\}\in(0,r)$, $K=1$, $\alpha= \sqrt{2}c/R^2$, and $\H=\{\eta_i=2^{i}c |i=1,\dots,N\}$ 
ensures
\begin{equation}
\label{thm7-for-dis}
\E\left[\sum_{t=1}^Tf_t(\hat{\x}_t)-\sum_{t=1}^Tf_t(\u_t)\right]=O\left((\sqrt{n}T^{3/4}+(n\bar{d})^{1/3}T^{2/3})\sqrt{P_T+1}\right)
\end{equation}
for any comparator sequence $\u_1,\dots,\u_T\in\K$.
\end{thm}
\begin{algorithm}[t]
\caption{Mild-BGD}
\label{Mild-BGD}
\begin{algorithmic}[1]
\STATE \textbf{Input:} parameters $\delta,\alpha$, a block size $K$ and a set $\H$ containing learning rates for experts
\STATE Activate experts $\left\{E^{\eta}|\eta\in\H\right\}$ by initializing $\y_1^\eta=\ze,\forall \eta\in\H$, and set $\tau=1,Z=T/K$
\STATE Sort learning rates as $\eta_1\leq\dots\leq\eta_{|\H|}$, and set $w_1^{\eta_i}=\frac{|\H|+1}{i(i+1)|\H|},\forall i\in[|\H|]$ 
\FOR{$z=1,\dots,Z$}
\STATE 
Compute ${\x}_z=\sum_{\eta\in\H}w_z^\eta\x_z^\eta$, where $\x_z^\eta=\y_{\tau}^\eta$ for each expert $E^{\eta}$, and set $\g_z=\ze,\mathcal{A}_z=\emptyset$
\FOR{$t=(z-1)K+1,\dots,zK$}
\STATE Play $\hat{\x}_t=\x_z+\delta\s_t$, where $\s_t\sim\mathcal{S}^n$, query $f_t(\hat{\x}_t)$, and receive $\{f_k(\hat{\x}_k)|k\in\F_t\}$
\FOR{$k\in\F_t$}
\STATE Set $z_k=\lfloor k/K\rfloor$ and update $\g_{z_k}=\g_{z_k}+\frac{n}{\delta} f_k(\hat{\x}_k) \mathbf{s}_k$
\STATE If $\g_{z_k}=\sum_{i=(z_k-1)K+1}^{z_kK}\frac{n}{\delta} f_i(\hat{\x}_i) \mathbf{s}_i$, set $\mathcal{A}_z=\mathcal{A}_z\cup\{z_k\}$
\ENDFOR
\ENDFOR
\STATE Update the weight of each expert $E^{\eta}$ as in \eqref{delayH-24-bandit}
\FOR{$k\in\mathcal{A}_z$ (in the ascending order)}
\STATE Compute $\y_{\tau+1}^\eta=\argmin_{\x\in\K_\delta}\|\x-(\y_{\tau}^\eta-\eta\g_k)\|_2^2$ for each expert $E^{\eta}$, and set $\tau=\tau+1$
\ENDFOR
\ENDFOR
\end{algorithmic}
\end{algorithm}
Compared with Corollaries \ref{cor2-bandit:label} and \ref{cor1-bandit:label}, Theorems \ref{thm3-bandit} and \ref{thm4-bandit} reduce the dependence of the dynamic regret on $P_T$ from linear to be sublinear. These two improved dynamic regret bounds now can recover not only the two best existing static regret bounds for the delayed bandit setting \citep{Bistritz-JMLR22,Wan-NIPS24}, but also the existing $O(\sqrt{n}T^{3/4}\sqrt{P_T+1})$ dynamic regret bound for the non-delayed bandit setting \citep{Zhao-JMLR22}. Moreover, it is also worth noting that the delay-dependent part of the dynamic regret bound in Theorem \ref{thm3-bandit} is $O(\sqrt{dT(1+P_T)})$. It actually matches the $\Omega(\sqrt{dT(P_T+1)})$ lower bound in Theorem \ref{thm-low}, which implies that the effect of delays on the dynamic regret in the bandit setting cannot be further improved in the worst case.
\begin{remark}
\emph{
Although $\beta^\prime=\sum_{z=1}^Zm_z$ and $\beta=\sum_{t=1}^Tm_t$ in Theorems \ref{thm3-bandit} and \ref{thm4-bandit} may be not available beforehand, this issue can be simply addressed by utilizing the doubling trick \citep{LEA97} as discussed in Remark \ref{rem4}. We omit the details of Mild-BGD with the doubling trick, and refer interested readers to Appendix \ref{appendixA} for Mild-OGD with the doubling trick, because Mild-BGD can be extended in the same way. Moreover, one may notice that Theorems \ref{thm2-bandit} and \ref{thm1-bandit} can be used to further improve the bounds in \eqref{thm6-for-dis} and \eqref{thm7-for-dis} if focusing on a specific $P_T$, and wonder whether such results can be extended to any sequence of comparators. However, this extension is highly non-trivial, and thus we leave it as a future work (see Appendix \ref{appendixB} for more detailed discussions).
}
\end{remark}

\section{Theoretical Analysis}
In this section, we provide the proofs of our theoretical guarantees on Mild-OGD, the lower bound, and Mild-BGD.

\subsection{Proof of Theorem \ref{thm1}}
Let $\tau_t=1+\sum_{i=1}^{t-1}|\F_i|$. Due to the convexity of loss functions, it is not hard to verify that
\begin{equation}
\label{thm1-eq1}
\begin{split}
&\sum_{t=1}^Tf_t(\x_t^\eta)-\sum_{t=1}^Tf_t(\u_t)
\leq\sum_{t=1}^T\left\langle \nabla f_t(\x_t^\eta),\x_t^\eta-\u_t\right\rangle\\
=&\sum_{t=1}^T\left\langle \nabla f_{c_t}(\x_{c_t}^\eta),\x_{c_t}^\eta-\u_{c_t}\right\rangle=\sum_{t=1}^T\left\langle \nabla f_{c_t}(\x_{c_t}^\eta),\y_{\tau_{c_t}}^\eta-\u_{c_t}\right\rangle\\
=&\sum_{t=1}^T\left\langle \nabla f_{c_t}(\x_{c_t}^\eta),\y_{t}^\eta-\u_{t}\right\rangle+\sum_{t=1}^T\left\langle \nabla f_{c_t}(\x_{c_t}^\eta),\y_{\tau_{c_t}}^\eta-\y_t^\eta\right\rangle+\sum_{t=1}^T\left\langle \nabla f_{c_t}(\x_{c_t}^\eta),\u_t-\u_{c_t}\right\rangle
\end{split}
\end{equation}
where the first equality is due to the fact that $c_1,\dots,c_T$ is a permutation of $1,\dots,T$, and the second equality is due to $\x_t^\eta=\y_{\tau_t}^\eta$ in Algorithm \ref{Ader-Expert}.

Let $\y^{\prime}_{t+1}=\y_{t}^\eta-\eta\nabla f_{c_t}(\x_{c_t}^\eta)$. For the first term in the right side of (\ref{thm1-eq1}), we have
\begin{equation}
\label{thm1-eq2}
\begin{split}
&\sum_{t=1}^T\left\langle \nabla f_{c_t}(\x_{c_t}^\eta),\y_{t}^\eta-\u_{t}\right\rangle=\sum_{t=1}^T\frac{\left\langle\y_{t}^\eta-\y_{t+1}^\prime,\y_t^\eta-\u_t\right\rangle}{\eta}\\
=&\sum_{t=1}^T\frac{\left(\|\y_{t}^\eta-\u_t\|_2^2-\|\y_{t+1}^\prime-\u_t\|_2^2+\|\y_{t}^\eta-\y_{t+1}^\prime\|_2^2\right)}{2\eta}\\
\leq&\sum_{t=1}^T\frac{\left(\|\y_{t}^\eta-\u_t\|_2^2-\|\y_{t+1}^\eta-\u_t\|_2^2+\|\y_{t}^\eta-\y_{t+1}^\prime\|_2^2\right)}{2\eta}\\
=&\sum_{t=1}^T\left(\frac{\left(\|\y_{t}^\eta\|_2^2-\|\y_{t+1}^\eta\|_2^2\right)}{2\eta}+\frac{\langle\y_{t+1}^\eta-\y_t^\eta,\u_t\rangle}{\eta}+\frac{\|\y_{t}^\eta-\y_{t+1}^\prime\|_2^2}{2\eta}\right)\\
\leq&\frac{1}{\eta}\langle\y_{T+1}^\eta,\u_T\rangle+\sum_{t=2}^T\frac{1}{\eta}\langle\u_{t-1}-\u_t,\y_t^\eta\rangle+\sum_{t=1}^T\frac{\|\y_{t}^\eta-\y_{t+1}^\prime\|_2^2}{2\eta}\\
\leq&\frac{1}{\eta}\|\y_{T+1}^\eta\|_2\|\u_T\|_2+\sum_{t=2}^T\frac{1}{\eta}\|\u_{t-1}-\u_t\|_2\|\y_t^\eta\|_2+\sum_{t=1}^T\frac{\eta\|\nabla f_{c_t}(\x_{c_t}^\eta)\|_2^2}{2}\\
\leq&\frac{R^2+RP_T}{\eta}+\frac{\eta TG^2}{2}
\end{split}
\end{equation}
where the second inequality is due to $\y_1^\eta=\ze$ and $\|\y_{T+1}^\eta\|_2^2\geq0$, and the last inequality is due to Assumptions \ref{assum1} and \ref{assum2}.


Next, we proceed to bound the second term in the right side of (\ref{thm1-eq1}). Let $q=c_t+d_{c_t}-1$. Note that before round $c_t\leq q$, Algorithm \ref{Ader-Expert} has received $\tau_{c_t}-1$ gradients, and thus has generated $\y_1^\eta,\dots,\y_{\tau_{c_t}}^\eta$. Since the gradient $\nabla f_{c_t}(\x_{c_t}^\eta)$ is used to update $\y_t^\eta$ in round $q$, we must have $\tau_{c_t}\leq t$ and it is not hard to verify that
\begin{equation}
\label{thm1-eq4}
\begin{split}
\sum_{t=1}^T\left\langle \nabla f_{c_t}(\x_{c_t}^\eta),\y_{\tau_{c_t}}^\eta-\y_t^\eta\right\rangle\leq&\sum_{t=1}^T\sum_{k=\tau_{c_t}}^{t-1}\left\|\y_{k}^\eta-\y_{k+1}^\eta\right\|_2\leq\sum_{t=1}^T\sum_{k=\tau_{c_t}}^{t-1}G\left\|\y_{k}^\eta-\y_{k+1}^\prime\right\|_2\\
=&\sum_{t=1}^T\sum_{k=\tau_{c_t}}^{t-1}\left\|\eta\nabla f_{c_k}(\x_{c_k}^\eta)\right\|_2\leq\eta G\sum_{t=1}^T\left(t-\tau_{c_t}\right)
\end{split}
\end{equation}
where both the last inequality is due to Assumption \ref{assum1}. 

Due to the definitions of $\tau_t$ and $m_t$, it holds that
\begin{equation}
\label{thm1-eq5}
\begin{split}
\sum_{t=1}^T\left(t-\tau_{c_t}\right)=\sum_{t=1}^Tt-\sum_{t=1}^T\tau_{c_t}=\sum_{t=1}^Tt-\sum_{t=1}^T\tau_{t}=\sum_{t=1}^Tm_t
\end{split}
\end{equation}
where the second equality is due to the fact that $c_1,\dots,c_T$ is a permutation of $1,\dots,T$. 

Combining (\ref{thm1-eq4}) with (\ref{thm1-eq5}) and Assumption \ref{assum1}, we have
\begin{equation}
\label{thm1-eq6}
\begin{split}
\sum_{t=1}^T\left\langle \nabla f_{c_t}(\x_{c_t}^\eta),\y_{\tau_{c_t}}^\eta-\y_t^\eta\right\rangle\leq \sum_{t=1}^TG\left\|\y_{\tau_{c_t}}^\eta-\y_t^\eta\right\|_2\leq\eta G^2\sum_{t=1}^Tm_t.
\end{split}
\end{equation}
Finally, combining (\ref{thm1-eq1}) with (\ref{thm1-eq2}), (\ref{thm1-eq6}), and Assumption \ref{assum1}, we have
\begin{align*}
\sum_{t=1}^Tf_t(\x_t^\eta)-\sum_{t=1}^Tf_t(\u_t)
\leq\frac{R^2+RP_T}{\eta}+{\eta G^2}\sum_{t=1}^T(m_t+1)+\sum_{t=1}^TG\|\u_t-\u_{c_t}\|_2
\end{align*}
which completes this proof.


\subsection{Proof of Lemma \ref{lem1_DOGD_main}}
Since $\g_{c_t}^\eta$ is the $t$-th used gradient and arrives at the end of round $c_{t}+d_{c_t}-1$, it is not hard to verify that
\begin{equation}
\label{order1}
t\leq c_{t}+d_{c_t}-1\leq c_t+d-1
\end{equation}
for any $t\in[T]$, and there are at most $t-1$ arrived gradients before round $c_t+d_{c_t}-1$. Note that gradients queried at rounds $1,\dots,t$ must have arrived at the end of round $t+d-1$. Therefore, we also have $c_t+d_{c_t}-2< t+d-1$, which implies that
\begin{equation}
\label{order2}
c_t\leq t+d-d_{c_t}\leq t+d-1.
\end{equation}
If $t\in[T]$ and $c_t\leq t$, according to (\ref{order1}), we have
\begin{equation}
\label{case1}
\begin{split}
\|\u_t-\u_{c_t}\|_2\leq\sum_{k=c_{t}}^{t-1}\|\u_{k+1}-\u_{k}\|_2\leq \sum_{k=c_{t}}^{\min\{c_t+d-2,T-1\}}\|\u_{k+1}-\u_{k}\|_2.
\end{split}
\end{equation}
Otherwise, if $t\in[T]$ and $c_t>t$, according to (\ref{order2}), we have
\begin{equation}
\label{case2}
\begin{split}
\|\u_t-\u_{c_t}\|_2\leq\sum_{k=t}^{c_t-1}\|\u_{k+1}-\u_{k}\|_2\leq \sum_{k=t}^{\min\{t+d-2,T-1\}}\|\u_{k+1}-\u_{k}\|_2.
\end{split}
\end{equation}
Combining (\ref{case1}) and (\ref{case2}), we have
\begin{equation*}
\begin{split}
\sum_{t=1}^T\|\u_t-\u_{c_t}\|_2\leq&\sum_{t=1}^T\sum_{k=c_{t}}^{\min\{c_t+d-2,T-1\}}\|\u_{k+1}-\u_{k}\|_2+\sum_{t=1}^T\sum_{k=t}^{\min\{t+d-2,T-1\}}\|\u_{k+1}-\u_{k}\|_2\\
=&2\sum_{t=1}^T\sum_{k=t}^{\min\{t+d-2,T-1\}}\|\u_{k+1}-\u_{k}\|_2\leq2\sum_{k=1}^{d-1}\sum_{t=1}^{T-1}\|\u_{t+1}-\u_{t}\|_2\leq 2dP_T
\end{split}
\end{equation*}
where the equality is due to the fact that $c_1,\dots,c_T$ is a permutation of $1,\dots,T$.

Finally, we complete this proof by noticing that Assumption \ref{assum2} and the definition of $K$ can ensure
\begin{equation*}
\begin{split}
\sum_{t=1}^T\|\u_t-\u_{c_t}\|_2=\sum_{t=1}^T\mathbb{I}(t\neq c_t)\|\u_t-\u_{c_t}\|_2\leq\sum_{t=1}^T2R\mathbb{I}(t\neq c_t)=2RK.
\end{split}
\end{equation*}

\subsection{Proof of Theorem \ref{thm2}}
Let $\eta_\ast=\sqrt{R(2R+P_T)/(\beta G^2)}$ denote an ideal learning rate, where $\beta=\sum_{t=1}^T(m_t+1)$. From Assumption \ref{assum2}, we have $0\leq P_T=\sum_{t=2}^T\|\u_{t}-\u_{t-1}\|_2\leq 2R(T-1)$, which implies that
\begin{equation*}
\eta_1=\frac{\sqrt{2}R}{G\sqrt{\beta}}\leq\eta_\ast\leq\frac{\sqrt{2T}R}{G\sqrt{\beta}}\leq\eta_{|\H|}.
\end{equation*}
Therefore, for any possible value of $P_T$, there must exist a learning rate $\eta_k\in\H$ such that
\begin{equation}
\label{eq3-thm1}
\eta_k\leq\eta_\ast\leq2\eta_{k}
\end{equation}
where $k=\lfloor\log_2\sqrt{(P_T+2R)/2R}\rfloor+1$.

Moreover, due to the convexity of loss functions and the definition of surrogate functions, i.e., $\ell_t(\x)=\langle\nabla f_t(\x_t),\x\rangle$, the dynamic regret of Algorithm \ref{Ader-Meta} can be upper bounded as follows
\begin{equation}
\label{surr-DR}
\begin{split}
\sum_{t=1}^Tf_t(\x_t)-\sum_{t=1}^Tf_t(\u_t)
\leq&\sum_{t=1}^T\langle\nabla f_t(\x_t),\x_t-\u_t\rangle\\
=&\sum_{t=1}^T\left(\ell_t\left(\x_t\right)-\ell_t\left(\x_t^{\eta_{k}}\right)\right)+\sum_{t=1}^T\left(\ell_t\left(\x_t^{\eta_{k}}\right)-\ell_t\left(\u_t\right)\right).
\end{split}
\end{equation}
To bound the first term in the right side of (\ref{surr-DR}), we introduce the following lemma.
\begin{lem}
\label{lem1}
Let $m_t=t-1-\sum_{i=1}^{t-1}|\F_i|$. Under Assumptions \ref{assum1} and \ref{assum2}, for any $\eta\in\H$, Algorithm \ref{Ader-Meta} ensures that
\begin{equation*}
\begin{split}
\sum_{t=1}^T\ell_t\left(\x_t\right)-\sum_{t=1}^T\ell_t(\x_t^\eta)
\leq\frac{1}{\alpha}\ln\frac{1}{w_1^\eta}+\alpha G^2R^2\sum_{t=1}^T(m_t+1).
\end{split}
\end{equation*}
\end{lem}
Combining Lemma \ref{lem1} with $(1/w_1^{\eta_k})\leq(k+1)^2$, $\alpha=1/(GR\sqrt{\beta})$, and $\beta=\sum_{t=1}^T(m_t+1)$, under Assumptions \ref{assum1} and \ref{assum2},
we have
\begin{equation}
\label{final-thm2-pre1}
\begin{split}
\sum_{t=1}^T\ell_t\left(\x_t\right)-\sum_{t=1}^T\ell_t(\x_t^{\eta_k})\leq&2GR\sqrt{\sum_{t=1}^T(m_t+1)}\ln(k+1)+GR\sqrt{\sum_{t=1}^T(m_t+1)}.
\end{split}
\end{equation}
Note that each expert $E^\eta$ actually is equal to running Algorithm \ref{Ader-Expert} with $\ell_1(\x),\dots,\ell_T(\x)$, where each gradient $\nabla \ell_t(\x_t^\eta)=\nabla f_t(\x_t)$ is delayed to the end of round $t+d_t-1$. Therefore, combining Theorem \ref{thm1} with Lemma \ref{lem1_DOGD_main} and the definition of $C$ in (\ref{bigC}), under Assumptions \ref{assum1} and \ref{assum2}, we have
\begin{equation*}
\begin{split}
\sum_{t=1}^T\ell_t\left(\x_t^{\eta_{k}}\right)-\sum_{t=1}^T\ell_t\left(\u_t\right)\leq&
\frac{R^2+RP_T}{\eta_k}+\eta_k G^2\sum_{t=1}^T(m_t+1) +C.
\end{split}
\end{equation*}
Then, combining the above inequality with (\ref{eq3-thm1}) and the definition of $\eta_\ast$, we have
\begin{equation}
\label{final-thm2-pre2}
\begin{split}
\sum_{t=1}^T\ell_t\left(\x_t^{\eta_{k}}\right)-\sum_{t=1}^T\ell_t\left(\u_t\right)
\leq&
\frac{2(R^2+RP_T)}{\eta_\ast}+\eta_\ast G^2\sum_{t=1}^T(m_t+1) +C\\
\leq&
3G\sqrt{R(2R+P_T)}\sqrt{\sum_{t=1}^T(m_t+1)}+C.
\end{split}
\end{equation}
Finally, we complete this proof by combining (\ref{surr-DR}) with \eqref{final-thm2-pre1}, \eqref{final-thm2-pre2}, and (\ref{thm01-eq6pre}).

\subsection{Proof of Lemma \ref{lem1}}
We start this proof by introducing some notations: $L_t^\eta=\sum_{i=1}^t\sum_{k\in\F_i}\ell_k(\x_k^\eta)$, $\c_t=(L_t^\eta)_{\eta\in\H}\in\mathbb{R}^{|\H|}$, $\w_{t}=(w_t^\eta)_{\eta\in\H}\in\mathbb{R}^{|\H|}$, $\tilde{L}_t^\eta=\sum_{i=1}^t\ell_i(\x_i^\eta)$, and $\tilde{\c}_t=(\tilde{L}_t^\eta)_{\eta\in\H}\in\mathbb{R}^{|\H|}$. 
For any $t\geq1$, according to Algorithm \ref{Ader-Meta}, it is easy to verify that
\begin{equation}
\label{lem1-eq1-pre}
w_{t+1}^\eta=\frac{w_t^\eta e^{-\alpha\sum_{k\in\F_t}\ell_{k}(\x_k^\eta)}}{\sum_{\mu\in\H}w_t^{\mu}e^{-\alpha\sum_{k\in\F_t}\ell_{k}(\x_k^\mu)}}=\frac{w_1^\eta e^{-\alpha L_t^\eta}}{\sum_{\mu\in\H}w_1^{\mu}e^{-\alpha L_t^\mu}}.
\end{equation}
Let $\Delta=\left\{\w\succeq\ze|\langle\w,\mathbf{1}\rangle=1\right\}$ and $\mathcal{R}(\w)=\sum_{i}w_i\ln w_i$. Due to \eqref{lem1-eq1-pre}, we also have
\begin{equation}
\label{lem1-eq1}
\w_{t+1}=\argmin_{\w\in\Delta}\left\langle-\frac{1}{\alpha}\ln(\w_1)+\c_t,\w\right\rangle+\frac{1}{\alpha}\mathcal{R}(\w).
\end{equation}
To analyze the effect of delays, we define a non-delayed counterpart of $\w_{t+1}$ as follows 
\begin{equation}
\label{lem3-rw-eq1-pre}
\tilde{\w}_{t+1}=(\tilde{w}_{t+1}^\eta)_{\eta\in\H}=\argmin_{\w\in\Delta}\left\langle-\frac{1}{\alpha}\ln(\w_1)+\tilde{\c}_t,\w\right\rangle+\frac{1}{\alpha}\mathcal{R}(\w)
\end{equation}
for any $t\geq1$, and define $\tilde{\w}_{1}=(\tilde{w}_{1}^\eta)_{\eta\in\H}=\w_{1}$. 

These new weights can result in an ideal weighted average of the experts, i.e, $\tilde{\x}_t=\sum_{\eta\in\H}\tilde{w}_t^\eta\x_t^\eta$. We notice that
\begin{equation}
\label{lem3-rw-eq1}
\sum_{t=1}^T\ell_t\left(\x_t\right)-\sum_{t=1}^T\ell_t(\x_t^\eta)=\sum_{t=1}^T\left(\ell_t\left(\x_t\right)-\ell_t(\tilde{\x}_t)\right)+\sum_{t=1}^T\left(\ell_t\left(\tilde{\x}_t\right)-\ell_t(\x_t^\eta)\right).
\end{equation}
To proceed this analysis, we first bound $\|\tilde{\x}_t-\x_t\|_2$ based on the following lemma.
\begin{lem}
\label{stability}
(Lemma 5 in \citet{DADO2011}) Let $\Pi_\K(\u,\alpha)=\argmin_{\x\in\K}\langle\u,\x\rangle+\frac{1}{\alpha}\mathcal{R}(\x)$. If $\mathcal{R}(\x)$ is $1$-strongly convex with respect to a norm $\|\cdot\|$, it holds that
\[\|\Pi_\K(\u,\alpha)-\Pi_\K(\v,\alpha)\|\leq\alpha\|\u-\v\|_\ast\]
for any $\u$ and $\v$, where $\|\cdot\|_\ast$ is the dual norm of $\|\cdot\|$.
\end{lem}
It is easy to verify that $\mathcal{R}(\w)=\sum_{i}w_i\ln w_i$ is $1$-strongly convex with respect to $\|\cdot\|_1$. Therefore, for any $t>1$, we can apply Lemma \ref{stability} to derive that
\begin{equation}
\label{lem1-eq4}
\begin{split}
\left\|\tilde{\x}_t-\x_t\right\|_2=&\left\|\sum_{\eta\in\H}(\tilde{w}_t^\eta-w_t^\eta)\x_t^\eta\right\|_2\leq\sum_{\eta\in\H}|\tilde{w}_t^\eta-w_t^\eta|\left\|\x_t^\eta\right\|_2\leq R\|\tilde{\w}_t-\w_t\|_1\\
\leq&\alpha R\|\tilde{\c}_{t-1}-\c_{t-1}\|_\infty\leq\alpha R\max_{\eta\in\H}\left|\sum_{k\in\U_{t-1}}\ell_k(\x_k^\eta)\right|\leq \alpha m_tGR^2
\end{split}
\end{equation}
where $\U_{t-1}=[t-1]\setminus\cup_{i\in[t-1]}\F_i$, and 
the last inequality is due to the definition of $\U_t$ and 
\begin{equation}
\label{Worst_Loss}
|\ell_k(\x_k^\eta)|=|\langle\nabla f_k(\x_k),\x_k^\eta\rangle|\leq \|\nabla f_k(\x_k)\|_2\|\x_k^\eta\|_2\leq GR,\forall k\in[T],\eta\in\H
\end{equation}
under Assumptions \ref{assum1} and \ref{assum2}. Then, it is easy to verify that
\begin{equation}
\label{lem3-rw-eq2}
\sum_{t=1}^T\left(\ell_t\left(\x_t\right)-\ell_t(\tilde{\x}_t)\right)=\sum_{t=1}^T\langle\nabla f_t(\x_t),\x_t-\tilde{\x}_t\rangle\leq \alpha G^2R^2\sum_{t=1}^Tm_t
\end{equation}
where the inequality is due to \eqref{lem1-eq4} and Assumption \ref{assum1}.

Now, we only need to analyze the last term in the right side of \eqref{lem3-rw-eq1}. To this end, we introduce the following lemma.
\begin{lem}
\label{lem-ftl}
(Lemma 6.6 in \citet{Garber16}) Let $\{h_t(\x)\}_{t=1}^T$ be a sequence of functions over a set $\K$, and let $\x_t^\ast\in\argmin_{\x\in\K}\sum_{i=1}^th_{i}(\x)$ for any $t\in[T]$. Then, it holds that \[\sum_{t=1}^Th_t(\x_t^\ast)-\min_{\x\in\K}\sum_{t=1}^Th_t(\x)\leq 0.\]
\end{lem}
Let $\mathbf{d}_t=(\ell_t(\x_t^\eta))_{\eta\in\H}\in\mathbb{R}^{|\H|}$ and $h_t(\x)=\langle\x,\mathbf{d}_t\rangle$ for any $t\in[T]$. Additionally, let $h_0(\x)=(\mathcal{R}(\w)-\langle\ln(\w_1),\w\rangle)/{\alpha}$ and $\Delta^\prime=\left\{\mathbf{e}_1,\dots,\mathbf{e}_{|\H|}\right\}\subset\Delta$, where $\mathbf{e}_i\in\mathbb{R}^{|\H|}$ and its $i$-th element equals to $1$. Combining with the previous definitions, it is not hard to verify that
\begin{equation}
\label{lem3-rw-eq3}
\begin{split}
\sum_{t=1}^T\left(\ell_t\left(\tilde{\x}_t\right)-\ell_t(\x_t^\eta)\right)=&\sum_{t=1}^T\left(\langle\tilde{\w}_{t},\mathbf{d}_t\rangle-\ell_t(\x_t^\eta)\right)= \sum_{t=1}^Th_t(\tilde{\w}_{t})-\sum_{t=1}^Th_t(\mathbf{e}_{\eta})\\
=&\sum_{t=1}^T\left(h_t(\tilde{\w}_{t+1})-h_t(\mathbf{e}_{\eta})\right)+\sum_{t=1}^T\left(h_t(\tilde{\w}_{t})-h_t(\tilde{\w}_{t+1})\right).
\end{split}
\end{equation}
Moreover, combining with \eqref{lem3-rw-eq1-pre} and $\tilde{\w}_1=\w_1$, we also have $\tilde{\w}_{t+1}=\argmin_{\w\in\Delta}\sum_{i=0}^{t}h_i(\w)$ for any $t\geq 0$. Therefore, we can apply Lemma \ref{lem-ftl} to derive that
\begin{equation}
\label{lem3-rw-eq4}
\begin{split}
\sum_{t=1}^T\left(h_t(\tilde{\w}_{t+1})-h_t(\mathbf{e}_{\eta})\right)\leq h_0(\mathbf{e}_{\eta})-h_0(\tilde{\w}_{1})=\frac{\langle\ln(\mathbf{e}_{\eta})-\ln(\w_1),\mathbf{e}_{\eta}\rangle}{\alpha}=\frac{1}{\alpha}\ln\frac{1}{w_1^\eta}.
\end{split}
\end{equation}
For the last term in the right side of \eqref{lem3-rw-eq3}, we have
\begin{equation}
\label{lem3-rw-eq5}
\begin{split}
&h_t(\tilde{\w}_{t})-h_t(\tilde{\w}_{t+1})=\langle\tilde{\w}_{t}-\tilde{\w}_{t+1},\mathbf{d}_t\rangle\leq\|\tilde{\w}_{t}-\tilde{\w}_{t+1}\|_1\|\mathbf{d}_t\|_{\infty}\\
\leq&\alpha \|\tilde{\c}_{t-1}-\tilde{\c}_{t}\|_\infty\|\mathbf{d}_t\|_{\infty}=\alpha \|\mathbf{d}_t\|_{\infty}^2=\alpha \left(\max_{\eta\in\H}\left|\ell_t(\x_t^\eta)\right|\right)^2\leq \alpha G^2R^2
\end{split}
\end{equation}
where the first inequality is due to Hölder's inequality, the second one is due to Lemma \ref{stability}, and the last one is due to \eqref{Worst_Loss}. By substituting \eqref{lem3-rw-eq4} and \eqref{lem3-rw-eq5} into \eqref{lem3-rw-eq3}, we have
\begin{equation}
\label{lem3-rw-eq6}
\begin{split}
\sum_{t=1}^T\left(\ell_t\left(\tilde{\x}_t\right)-\ell_t(\x_t^\eta)\right)\leq \frac{1}{\alpha}\ln\frac{1}{w_1^\eta}+\alpha TG^2R^2.
\end{split}
\end{equation}
Finally, we complete this proof by substituting \eqref{lem3-rw-eq2} and \eqref{lem3-rw-eq6} into \eqref{lem3-rw-eq1}.

\subsection{Proof of Theorem \ref{thm-low}}
Inspired by the proof of the lower bound in the non-delayed setting \citep{Zhang18_ader}, we first need to
establish a lower bound of static regret in the delayed setting. Although the seminal work of \citet{Weinberger02_TIT} has already provided such a lower bound, it only holds in the special case that $d$ divides $T$. To address this limitation, we establish a lower bound of static regret for any $d$ and $T$, which is presented in the following lemma.
\begin{lem}
\label{lem1_static_lower}
Suppose $\K=\left[-R/\sqrt{n},R/\sqrt{n}\right]^n$ which satisfies Assumption \ref{assum2}. For any OCO algorithm and any positive integer $d$, there exists a sequence of functions $f_1(\x),\dots,f_T(\x)$ satisfying Assumption \ref{assum1} and a sequence of delays $1\leq d_1,\dots,d_T\leq d$ such that
\begin{equation*}
\sum_{t=1}^Tf_t(\x_t)-\min_{\x\in\K}\sum_{t=1}^Tf_t(\x)\geq\frac{RGT}{\sqrt{2\left\lceil T/d\right\rceil}}.
\end{equation*}
\end{lem}
Let $Z=\lceil T/L\rceil$. Due to $L=\left\lceil 2TR/\max\{P,2R\}\right\rceil$, if $P<2R$, it is easy to verify that $L=T$ and $Z=1$. Otherwise, we have $P\in[2R,2TR]$, which implies that
\[
    Z\leq \frac{T}{L}+1=\frac{T}{\left\lceil 2TR/P\right\rceil}+1\leq \frac{P}{2R}+1.
\]
Combining these two cases, it must hold that $Z-1\leq P/(2R)$.

Then, we divide the total $T$ rounds into $Z$ blocks, where the length of the first $Z-1$ blocks is $L$ and that of the last block is $T-(Z-1)L$. In this way, we can define the set of rounds in the block $z$ as $\T_z=\{(z-1)L+1,\dots,\min\{zL,T\}\}$. Moreover, we define the feasible set of $\u_1,\dots,\u_T$ as
\[\C(P)=\left\{\u_1,\dots,\u_T\in\K\left|\sum_{t=2}^T\|\u_t-\u_{t-1}\|_2\leq P\right.\right\}\]
and construct a subset of $\C(P)$ as
\begin{align*}
\C^\prime(P)=&\left\{\u_1,\dots,\u_T\in\K\left|\u_{i}=\u_{j}, \forall z\in[Z],i,j\in\T_z\right.\right\}.
\end{align*}
Note that the connection $\C^\prime(P)\subseteq \C(P)$ is derived by the fact that the comparator sequence in $\C^\prime(P)$ only changes $Z-1\leq P/(2R)$ times, and thus its path-length does not exceed $P$.

Because of $\C^\prime(P)\subseteq \C(P)$ and Lemma \ref{lem1_static_lower}, there exists a sequence of functions $f_1(\x),\dots,f_T(\x)$ satisfying Assumption \ref{assum1} and a sequence of delays $1\leq d_1,\dots,d_T\leq d$ such that
\begin{equation*}
\begin{split}
&\sum_{t=1}^Tf_t(\x_t)-\min_{\u_1,\dots,\u_T\in\C(P)}\sum_{t=1}^Tf_t(\u_t)\geq\sum_{t=1}^Tf_t(\x_t)-\min_{\u_1,\dots,\u_T\in\C^\prime(P)}\sum_{t=1}^Tf_t(\u_t)\\
=&\sum_{z=1}^Z\left(\sum_{t\in\T_z}f_t(\x_t)-\min_{\x\in\K}\sum_{t\in\T_z}f_t(\x)\right)\geq\sum_{z=1}^Z\frac{RG|\T_z|}{\sqrt{2\left\lceil |\T_z|/d\right\rceil}}.
\end{split}
\end{equation*}
Finally, we can complete this proof by further noticing that
\begin{equation*}
\begin{split}
\sum_{z=1}^Z\frac{RG|\T_z|}{\sqrt{2\left\lceil |\T_z|/d\right\rceil}}\geq&\sum_{z=1}^Z\frac{RG|\T_z|}{\sqrt{2\left\lceil L/d\right\rceil}}=\frac{RGT}{\sqrt{2\left\lceil L/d\right\rceil}}\geq\left\{
\begin{aligned}
&\frac{RGT}{\sqrt{2}},~\text{if $d> L$;}\\
&\frac{G\sqrt{dR\max\{P,2R\}T}}{4},~\text{otherwise;}
\end{aligned}
\right.
\end{split}
\end{equation*}
where the first inequality is due to $|\T_z|\leq L$ for any $z\in[Z]$, and the last inequality is mainly due to
\begin{align*}
\left\lceil L/d\right\rceil\leq&2L/d=(2/d)\left\lceil 2TR/\max\{P,2R\}\right\rceil\leq 8TR/(\max\{P,2R\}d)
\end{align*}
for $d\leq L$.

\subsection{Proof of Lemma \ref{lem1_static_lower}}
Let $Z=\left\lceil T/d\right\rceil$. We first divide the total $T$ rounds into $Z$ blocks, where the length of the first $Z-1$ blocks is $d$ and that of the last block is $T-(Z-1)d$. In this way, we can define the set of rounds in the block $z$ as $\T_z=\{(z-1)d+1,\dots,\min\{zd,T\}\}$. For any $z\in[Z]$ and $t\in\T_z$, we construct the delay as $d_t=\min\{zd,T\}-t+1$, which satisfies $1\leq d_t\leq d$.
These delays ensure that the information of any function in each block $z$ is delayed to the end of the block, which is critical for us to construct loss functions that maximize the impact of delays on the static regret. 

Note that to establish the lower bound of the static regret in the non-delayed setting, one can utilize a randomized strategy to select loss functions for each round \citep{Abernethy08}. Here, to maximize the impact of delays, we only select one loss function $h_z(\x)$ for all rounds in the same block $z$, i.e., $f_t(\x)=h_z(\x)$ for any $t\in\T_z$. Specifically, we set $h_z(\x)=\frac{G}{\sqrt{n}}\langle\w_z,\x\rangle$, where the $i$-th coordinate of $\w_z$ is $\pm 1$ with probability $1/2$ for any $i\in[n]$ and will be denoted as $w_{z,i}$. It is not hard to verify that $h_z(\x)$ satisfies Assumption \ref{assum1}.

From the above definitions and $\reg(T)=\sum_{t=1}^Tf_t(\x_t)-\min_{\x\in\K}\sum_{t=1}^Tf_t(\x)$, we have
\begin{equation*}
\begin{split}
\E_{\w_1,\dots,\w_Z}\left[\reg(T)\right]=&\E_{\w_1,\dots,\w_Z}\left[\sum_{z=1}^Z\sum_{t\in\T_z}\frac{G}{\sqrt{n}}\langle\w_z,\x_t\rangle-\min_{\x\in\K}\sum_{z=1}^Z\sum_{t\in\T_z}\frac{G}{\sqrt{n}}\langle\w_z,\x\rangle\right]\\
=&\E_{\w_1,\dots,\w_Z}\left[-\min_{\x\in\K}\sum_{z=1}^Z\frac{G|\T_z|}{\sqrt{n}}\langle\w_z,\x\rangle\right]
\end{split}
\end{equation*}
where the third equality is due to $\E_{\w_1,\dots,\w_Z}[\langle\w_z,\x_t\rangle]=0$ for any $t\in\T_z$, which can be derived by the fact that any decision $\x_t$ in the block $z$ is made before receiving the information of $\w_z$, and thus is independent with $\w_z$. 

Since a linear function is minimized at the vertices of the cube, the above equality implies that
\begin{equation*}
\begin{split}
\E_{\w_1,\dots,\w_Z}\left[\reg(T)\right]=&-\E_{\w_1,\dots,\w_Z}\left[\min_{\x\in\{-R/\sqrt{n},R/\sqrt{n}\}^n}\sum_{z=1}^Z\frac{G|\T_z|}{\sqrt{n}}\langle\w_z,\x\rangle\right]\\
=&\E_{\w_1,\dots,\w_Z}\left[\sum_{i=1}^n\frac{R}{\sqrt{n}}\left|\sum_{z=1}^Z\frac{w_{z,i}G|\T_z|}{\sqrt{n}}\right|\right]={RG}\E_{\w_1,\dots,\w_Z}\left[\left|\sum_{z=1}^Zw_{z,1}|\T_z|\right|\right]\\
\geq&\frac{RG}{\sqrt{2}}\sqrt{\sum_{z=1}^Z|\T_z|^2}\geq\frac{RG}{\sqrt{2}}\sqrt{\frac{(\sum_{z=1}^Z|\T_z|)^2}{Z}}=\frac{RGT}{\sqrt{2\left\lceil T/d\right\rceil}}
\end{split}
\end{equation*}
where the first inequality is due to the Khintchine inequality, and the second one is due to the Cauchy-Schwarz inequality.

This expected lower bound implies that for
any OCO algorithm and any positive integer $d$, there exists a particular choice of $\w_1,\dots,\w_Z$ such that
\begin{equation*}
\begin{split}
\reg(T)\geq\frac{RGT}{\sqrt{2\left\lceil T/d\right\rceil}}.
\end{split}
\end{equation*}

\subsection{Proof of Theorems \ref{thm1-bandit} and \ref{thm2-bandit}}
Since Theorem \ref{thm1-bandit} can be exactly recovered by Theorem \ref{thm2-bandit} with $K=1$, we only need to prove Theorem \ref{thm2-bandit}. Specifically, we start this proof by noticing that
\begin{equation}
\label{thm2-bandit-eq1}
\begin{split}
\sum_{t=1}^Tf_t(\hat{\x}_t^\eta)-\sum_{t=1}^Tf_t(\u_t)
=&\sum_{z=1}^Z\sum_{t=q_z}^{q_{z+1}-1}\left(f_t(\hat{\x}_t^\eta)-f_t(\u_{q_z})\right)+\sum_{z=1}^Z\sum_{t=q_z}^{q_{z+1}-1}\left(f_t(\u_{q_z})-f_t(\u_t)\right)\\
\leq&\sum_{z=1}^Z\sum_{t=q_z}^{q_{z+1}-1}\left(f_t(\hat{\x}_t^\eta)-f_t(\u_{q_z})\right)+\sum_{z=1}^Z\sum_{t=q_z}^{q_{z+1}-1}G\|\u_{q_z}-\u_t\|_2
\end{split}
\end{equation}
where the inequality is due to Assumption \ref{assum1}.

For the last term in the right side of \eqref{thm2-bandit-eq1}, it is easy to derive the following upper bound.
\begin{lem}
\label{lem1-thm2-bandit}
Let $q_z=(z-1)K+1$. Under Assumption \ref{assum-b2}, we have
\[
    \sum_{z=1}^Z\sum_{t=q_z}^{q_{z+1}-1}G\|\u_{q_z}-\u_t\|_2\leq G\sqrt{2(K-1)RTP_T}.
\]
\end{lem}
Then, to analyze the first term in the right side of \eqref{thm2-bandit-eq1}, we introduce some useful properties about the $\delta$-smoothed version of loss functions.
\begin{lem}
\label{estimator-lem2}
(Lemma 2.6 in \citet{Hazan2016} and Lemma 6 in \citet{DBOCG-SC})
Let $f(\x):\mathbb{R}^n\to\mathbb{R}$ be convex and $G$-Lipschitz over a convex set $\K$ satisfying Assumption \ref{assum-b2}. For any $\delta\in(0,r)$, its $\delta$-smoothed version defined as $\hat{f}_\delta(\mathbf{x})=\mathbb{E}_{\mathbf{s}\sim\mathcal{B}^n}\left[f(\mathbf{x}+\delta\mathbf{s})\right]$ has the following properties: i) $\hat{f}_\delta(\x)$ is convex over $\K_\delta$; ii) $|\hat{f}_\delta(\x)-f(\x)|\leq\delta G$ for any $\x\in\K_\delta$; iii) $\hat{f}_\delta(\x)$ is $G$-Lipschitz over $\K_\delta$.
\end{lem} 
Let $\hat{f}_{t,\delta}(\x)$ denote the $\delta$-smoothed version of $f_t(\x)$, and let $\hat{\u}_t=(1-\delta/r)\mathbf{u}_t$ denote the shrunk comparator. Due to Assumption \ref{assum1}, it is not hard to verify that
\begin{equation}
\label{thm2-bandit-eq2}
\begin{split}
\sum_{z=1}^Z\sum_{t=q_z}^{q_{z+1}-1}\left(f_t(\hat{\x}_t^\eta)-f_t(\u_{q_z})\right)\leq &\sum_{z=1}^Z\sum_{t=q_z}^{q_{z+1}-1}\left(f_t(\x_{z}^\eta)-f_t(\hat{\u}_{q_z})+\delta G\|\s_t\|_2+\frac{\delta G}{r}\|\u_{q_z}\|_2\right)\\
\leq &\sum_{z=1}^Z\sum_{t=q_z}^{q_{z+1}-1}(f_t(\x_{z}^\eta)-f_t(\hat{\u}_{q_z}))+\delta GT+\frac{\delta GRT}{r}\\
\leq &\sum_{z=1}^Z\sum_{t=q_z}^{q_{z+1}-1}(\hat{f}_{t,\delta}({\x}_z^\eta)-\hat{f}_{t,\delta}(\hat{\u}_{q_z}))+3\delta GT+\frac{\delta GRT}{r}\\
\leq & \sum_{z=1}^Z\sum_{t=q_z}^{q_{z+1}-1}\langle\nabla \hat{f}_{t,\delta}({\x}_z^\eta),{\x}_z^\eta-\hat{\u}_{q_z}\rangle+3\delta GT+\frac{\delta GRT}{r}
\end{split}
\end{equation}
where the second inequality is due to $\s_t\sim\mathcal{S}^n$ and Assumption \ref{assum-b2}, the third inequality is due to the second property in Lemma \ref{estimator-lem2}, and the last inequality is due to the first property in Lemma \ref{estimator-lem2}. 

Next, we proceed to bound the first term in the right side of \eqref{thm2-bandit-eq2} with an idea similar to the proof of Theorem \ref{thm1}. Note that Algorithm \ref{BDBGD} ensures $\x_z^\eta=\y_{\tau_z}^\eta$, where $\tau_z=1+\sum_{i=1}^{z-1}|\mathcal{A}_i|$, and $c_1,\dots,c_Z$ is a permutation of $1,\dots,Z$. Therefore, it is not hard to verify that
\begin{equation}
\label{thm2-bandit-eq3}
\begin{split}
&\sum_{z=1}^Z\sum_{t=q_z}^{q_{z+1}-1}\langle\nabla \hat{f}_{t,\delta}({\x}_z^\eta),{\x}_z^\eta-\hat{\u}_{q_z}\rangle=\sum_{z=1}^Z\sum_{t=q_{c_z}}^{q_{c_z+1}-1}\langle\nabla \hat{f}_{t,\delta}({\x}_{c_z}^\eta),\y_{\tau_{c_z}}^\eta-\hat{\u}_{q_{c_z}}\rangle\\
=&\sum_{z=1}^Z\sum_{t=q_{c_z}}^{q_{c_z+1}-1}\langle\nabla \hat{f}_{t,\delta}({\x}_{c_z}^\eta),{\y}_{z}^\eta-\hat{\u}_{q_z}\rangle+\sum_{z=1}^Z\sum_{t=q_{c_z}}^{q_{c_z+1}-1}\langle\nabla \hat{f}_{t,\delta}({\x}_{c_z}^\eta),\y_{\tau_{c_z}}^\eta-{\y}_z^\eta+\hat{\u}_{q_z}-\hat{\u}_{q_{c_z}}\rangle\\
\leq& \sum_{z=1}^Z\sum_{t=q_{c_z}}^{q_{c_z+1}-1}\langle\nabla \hat{f}_{t,\delta}({\x}_{c_z}^\eta),{\y}_{z}^\eta-\hat{\u}_{q_z}\rangle+\sum_{z=1}^ZKG\|\y_{\tau_{c_z}}^\eta-{\y}_z^\eta\|_2+\sum_{z=1}^ZKG\|\hat{\u}_{q_z}-\hat{\u}_{q_{c_z}}\|_2
\end{split}
\end{equation}
where the inequality is due to the third property in Lemma \ref{estimator-lem2} and Assumption \ref{assum1}.

Due to Lemma \ref{estimator-lem1}, we have
\begin{equation}
\label{thm2-bandit-eq4}
\begin{split}
\mathbb{E}\left[\sum_{z=1}^Z\sum_{t=q_{c_z}}^{q_{c_z+1}-1}\langle\nabla \hat{f}_{t,\delta}({\x}_{c_z}^\eta),{\y}_{z}^\eta-\hat{\u}_{q_z}\rangle\right]=&\mathbb{E}\left[\sum_{z=1}^Z\sum_{t=q_{c_z}}^{q_{c_z+1}-1}\left\langle \frac{n}{\delta} f_t({\x}_{c_z}^\eta+\delta\mathbf{s}_{t}) \mathbf{s}_{t},{\y}_{z}^\eta-\hat{\u}_{q_z}\right\rangle\right]\\
=&\mathbb{E}\left[\sum_{z=1}^Z\left\langle \g_{c_z}^\eta,{\y}_{z}^\eta-\hat{\u}_{q_z}\right\rangle\right].
\end{split}
\end{equation}
Recall that Algorithm \ref{BDBGD} ensures $\y_{z+1}^\eta=\argmin_{\x\in\K_\delta}\|\x-(\y_{z}^\eta-\eta\g_{c_z}^\eta)\|_2^2$. Following \eqref{thm1-eq2}, it is easy to verify that 
\begin{equation}
\label{thm2-bandit-eq5}
\begin{split}
\sum_{z=1}^Z\left\langle \g_{c_z}^\eta,{\y}_{z}^\eta-\hat{\u}_{q_z}\right\rangle\leq&\frac{R^2+R\hat{P}_Z}{\eta}+\frac{\eta }{2}\sum_{z=1}^Z\|\g_{c_z}^\eta\|_2^2
\end{split}
\end{equation}
where $\hat{P}_Z=\sum_{z=2}^Z\|\hat{\u}_{q_z}-\hat{\u}_{q_{z-1}}\|_2$. 

The right side of \eqref{thm2-bandit-eq5} can be further bounded by introducing the following lemma regarding the cumulative estimated gradients of all rounds in any block $z\in[Z]$.
\begin{lem}
\label{block-gradient}
Under Assumptions \ref{assum1} and \ref{assum-b1}, for any $z\in[Z]$, Algorithm \ref{BDBGD} ensures 
\[
\E\left[\|\g_z^{\eta}\|_2^2\right]=\E\left[\left\|\sum_{i=(z-1)K+1}^{zK}\frac{n}{\delta} f_i({\x}_z^\eta+\delta\mathbf{s}_i) \mathbf{s}_i\right\|_2^2\right]\leq K\left(\frac{nM}{\delta}\right)^2+(K^2-K)G^2.
\]
\end{lem}
Therefore, combining \eqref{thm2-bandit-eq4} with \eqref{thm2-bandit-eq5} and Lemma \ref{block-gradient}, we have
\begin{equation}
\label{thm2-bandit-eq4-plus}
\begin{split}
\mathbb{E}\left[\sum_{z=1}^Z\sum_{t=q_{c_z}}^{q_{c_z+1}-1}\langle\nabla \hat{f}_{t,\delta}({\x}_{c_z}^\eta),{\y}_{z}^\eta-\hat{\u}_{q_z}\rangle\right]\leq \frac{R^2+R\hat{P}_Z}{\eta}+\frac{\eta \gamma Z}{2}
\end{split}
\end{equation}
where $\gamma=K({nM}/{\delta})^2+(K^2-K)G^2$.

Now, we only need to consider the second term in the right side of \eqref{thm2-bandit-eq3}. We notice that before block $c_z$, Algorithm \ref{BDBGD} has utilized the cumulative estimated gradients of blocks in the set $\cup_{i=1}^{c_z-1}\mathcal{A}_i$ to generated $\y_1^\eta,\dots,\y_{\tau_{c_z}}^\eta$. Moreover, the cumulative estimated gradients $\g_{c_z}^\eta$ will only be used to update $\y_z^\eta$ at the end of a certain block $c\geq c_z$. Thus, we have $\tau_{c_z}\leq z$, and can verify that
\begin{equation}
\label{thm2-bandit-eq7}
\begin{split}
\mathbb{E}\left[\sum_{z=1}^Z\|\y_{\tau_{c_z}}^\eta-{\y}_z^\eta\|_2\right]\leq& \mathbb{E}\left[\sum_{z=1}^Z\sum_{k=\tau_{c_z}}^{z-1}\|\y_{k}^\eta-{\y}_{k+1}^\eta\|_2\right]\leq \mathbb{E}\left[\sum_{z=1}^Z\sum_{k=\tau_{c_z}}^{z-1}\|\eta\g_{c_k}^\eta\|_2\right]\\
\leq&\sum_{z=1}^Z\sum_{k=\tau_{c_z}}^{z-1}\sqrt{\gamma}\eta=\sqrt{\gamma}\eta\sum_{z=1}^Z(z-\tau_{c_z})=\sqrt{\gamma}\eta\sum_{z=1}^Zm_z
\end{split}
\end{equation}
where the third inequality is due to Lemma \ref{block-gradient} and Jensen's inequality, and the last equality can be derived by simply following \eqref{thm1-eq5}.

Combining \eqref{thm2-bandit-eq3} with \eqref{thm2-bandit-eq4-plus} and \eqref{thm2-bandit-eq7}, we have
\begin{equation}
\label{thm2-bandit-eq-final-pre}
\begin{split}
\mathbb{E}\left[\sum_{z=1}^Z\sum_{t=q_z}^{q_{z+1}-1}\langle\nabla \hat{f}_{t,\delta}({\x}_z^\eta),{\x}_z^\eta-\hat{\u}_{q_z}\rangle\right]\leq &\frac{R^2+R\hat{P}_Z}{\eta}+\frac{\eta \gamma Z}{2}+\sqrt{\gamma}\eta KG\sum_{z=1}^Zm_z\\
&+\sum_{z=1}^ZKG\|\hat{\u}_{q_z}-\hat{\u}_{q_{c_z}}\|_2.
\end{split}
\end{equation}
Due to \eqref{thm2-bandit-eq2}, \eqref{thm2-bandit-eq-final-pre} and Lemma \ref{lem1-thm2-bandit}, we can take expectation on the two sides of \eqref{thm2-bandit-eq1}, and derive that
\begin{equation}
\label{thm2-bandit-eq-final}
\begin{split}
\mathbb{E}\left[\sum_{t=1}^Tf_t(\hat{\x}_t^\eta)-\sum_{t=1}^Tf_t(\u_t)\right]\leq& \frac{R^2+R\hat{P}_Z}{\eta}+\frac{\eta\gamma Z}{2}+\sqrt{\gamma}\eta KG\sum_{z=1}^Zm_z+\sum_{z=1}^ZKG\|\hat{\u}_{q_z}-\hat{\u}_{q_{c_z}}\|_2\\
&+3\delta GT+\frac{\delta GRT}{r}+G\sqrt{2(K-1)RTP_T}.
\end{split}
\end{equation}
Finally, we also notice that 
\begin{equation}
\label{convert-path} 
\|\hat{\u}_i-\hat{\u}_j\|_2=(1-\delta/r)\|{\u}_i-{\u}_j\|_2\leq \|{\u}_i-{\u}_j\|_2
\end{equation}
for any $i,j\in[T]$, which implies that
\begin{equation}
\label{convert-from-block-path}
    \hat{P}_Z\leq \sum_{z=2}^Z\|{\u}_{q_z}-{\u}_{q_{z-1}}\|_2\leq \sum_{z=2}^Z\sum_{i=q_{z-1}+1}^{q_z}\|{\u}_{i}-{\u}_{i-1}\|_2\leq P_T.
\end{equation}
It is easy to complete this proof by combining \eqref{thm2-bandit-eq-final} with \eqref{convert-path} and \eqref{convert-from-block-path}.

\subsection{Proof of Corollary \ref{cor1-bandit:label}}
Due to the values of $\eta$ and $\delta$, it is easy to verify that
\begin{equation*}
\begin{split}
&\frac{R^2+RP_T}{\eta}+\frac{\eta Tn^2M^2}{2\delta^2}+\frac{\eta nGM}{\delta}\sum_{t=1}^Tm_t\\
\leq&(R+RP_T)\left(\sqrt{n}T^{3/4}+\left(nT\sum_{t=1}^Tm_t\right)^{1/3}\right)+\frac{\sqrt{n}T^{3/4}RM}{2}+\left(nT\sum_{t=1}^Tm_t\right)^{1/3}GR.
\end{split}
\end{equation*}
Moreover, due to the value of $\delta$ and Lemma \ref{lem1_DOGD_main}, we have
\begin{equation*}
\begin{split}
\sum_{t=1}^TG\left\|{\u}_t-{\u}_{c_t}\right\|_2+
3\delta GT+\frac{\delta GRT}{r}\leq&C+\left(3G+\frac{GR}{r}\right)\left(\sqrt{n}T^{3/4}+\left(nT\sum_{t=1}^Tm_t\right)^{1/3}\right).
\end{split}
\end{equation*}
Finally, combining Theorem \ref{thm1-bandit} with the above two inequalities and \eqref{thm01-eq6pre}, we have
\[
   \E\left[\sum_{t=1}^Tf_t(\hat{\x}_t^\eta)-\sum_{t=1}^Tf_t(\u_t)\right]=O\left((\sqrt{n}T^{3/4}+(n\bar{d})^{1/3}T^{2/3})(P_T+1)+C\right).
\]

\subsection{Proof of Lemma \ref{lem1-thm2-bandit}}
Due to Assumption \ref{assum-b2}, it is easy to verify that $\sum_{z=1}^Z\sum_{t=q_z}^{q_{z+1}-1}G\|\u_{q_z}-\u_t\|_2\leq 2GTR$. Moreover, due to the definition of $q_z$, we also have
\begin{equation*}
\begin{split}
\sum_{z=1}^Z\sum_{t=q_z}^{q_{z+1}-1}G\|\u_{q_z}-\u_t\|_2\leq \sum_{z=1}^Z\sum_{t=q_z+1}^{q_{z+1}-1}(K-1)G\|\u_{t}-\u_{t-1}\|_2\leq (K-1)GP_T.
\end{split}
\end{equation*}
Finally, we can complete this proof by simply combining these two upper bounds.
\subsection{Proof of Lemma \ref{block-gradient}}
This lemma can be proved by simply following the proof of Lemma 5 in \citet{Garber19}. Here, we include the detailed proof for completeness. For brevity, let $\mathcal{T}_z=\{(z-1)K+1,\dots,zK\}$ and $\nabla_i=\frac{n}{\delta} f_i({\x}_z^\eta+\delta\mathbf{s}_i) \mathbf{s}_i$ for any $i\in\mathcal{T}_z$. It is easy to verify that
\begin{align*}
\E\left[\|\g_z^{\eta}\|_2^2\right]=&\E\left[\left\|\sum_{i\in\mathcal{T}_z}\nabla_i\right\|_2^2\right]=\E\left[\sum_{i\in\mathcal{T}_z}\left\|\nabla_i\right\|_2^2+\sum_{i,j\in\mathcal{T}_z,i\neq j}\langle\nabla_i,\nabla_j\rangle\right]\\
\leq&|\mathcal{T}_z|\left(\frac{nM}{\delta}\right)^2+\E\left[\sum_{i,j\in\mathcal{T}_z,i\neq j}\langle\E[\nabla_i|\x_z^\eta],\E[\nabla_j|\x_z^\eta]\rangle\right]\\
\leq&|\mathcal{T}_z|\left(\frac{nM}{\delta}\right)^2+\E\left[\sum_{i,j\in\mathcal{T}_z,i\neq j}\|\E[\nabla_i|\x_z^\eta]\|_2\|\E[\nabla_j|\x_z^\eta]\|_2\right]\\
\leq&|\mathcal{T}_z|\left(\frac{nM}{\delta}\right)^2+(|\mathcal{T}_z|^2-|\mathcal{T}_z|)G^2= K\left(\frac{nM}{\delta}\right)^2+(K^2-K)G^2
\end{align*}
where the first inequality is due to Assumption \ref{assum-b1}, and the third inequality is due to Assumption \ref{assum1}, Lemma \ref{estimator-lem1}, and the third property in Lemma \ref{estimator-lem2}.

\subsection{Proof of Corollary \ref{cor2-bandit:label}}
\label{sub:proof_of_corollary_ref_cor2_bandit_label}
Before applying Theorem \ref{thm2-bandit}, we need to analyze $\sum_{z=1}^Zm_z$ and $\sum_{z=1}^Z\|{\u}_{q_z}-{\u}_{q_{c_z}}\|_2$ more carefully. Let $\hat{d}_1,\dots,\hat{d}_z$ denote the block-level delays, i.e., the block index $z$ is included in the set $\mathcal{A}_{z+\hat{d}_z-1}$, and let $\hat{d}=\max\{\hat{d}_1,\dots,\hat{d}_z\}$. We first notice that for any $z\in[Z]$, the loss values $f_{1}(\hat{\x}_{1}^\eta),\dots,f_{zK}(\hat{\x}_{zK}^\eta)$ must be available at the end of round $zK+d-1$. 
Therefore, according to the block index of round $zK+d-1$, it is easy to verify that
\begin{equation}
\label{block_max_d}
\hat{d}_z=\left\lfloor\frac{zK+d-1}{K}\right\rfloor-z+1\leq  \frac{d}{K}+1.
\end{equation}
Similar to \eqref{thm01-eq6pre}, we can use \eqref{block_max_d} to derive the following bound
\begin{equation}
\label{block-level-delay-bound}
\sum_{z=1}^Zm_z\leq\sum_{z=1}^Z\left(\hat{d}_z-1\right)\leq \frac{dZ}{K}= \frac{dT}{K^2}.
\end{equation}
Then, to further bound the term $\sum_{z=1}^Z\|{\u}_{q_z}-{\u}_{q_{c_z}}\|_2$, we need to bound the gap between $z$ and $c_z$. First, due to the fact that $\g_{c_z}^\eta$ is used to update $\y_z^\eta$ at the end of block $c_z+\hat{d}_{c_z}-1$, it is not hard to verify that
\begin{equation}
\label{order1-block}
z\leq c_{z}+\hat{d}_{c_z}-1\leq c_z+\hat{d}-1
\end{equation}
for any $z\in[Z]$. Second, we also notice that there are at most $z-1$ cumulative estimated gradients, which are available and used before block $c_z+\hat{d}_{c_z}-1$. Since $\g_{1}^\eta,\dots,\g_{z}^\eta$ must have been available at the end of block $z+\hat{d}-1$, we have $c_z+\hat{d}_{c_z}-2< z+\hat{d}-1$, which implies that
\begin{equation}
\label{order2-block}
c_z\leq z+\hat{d}-\hat{d}_{c_z}\leq z+\hat{d}-1.
\end{equation}
If $z\in[Z]$ and $c_z\leq z$, according to \eqref{order1-block}, we have
\begin{equation}
\label{case1-block}
\begin{split}
\|{\u}_{q_z}-{\u}_{q_{c_z}}\|_2\leq \sum_{k={c_z}}^{z-1}\|\u_{q_{k+1}}-\u_{q_{k}}\|_2\leq \sum_{k={c_z}}^{\min\{c_z+\hat{d}-2,Z-1\}}\|\u_{q_{k+1}}-\u_{q_{k}}\|_2.
\end{split}
\end{equation}
Otherwise, if $z\in[Z]$ and $c_z> z$, according to \eqref{order2-block}, we have
\begin{equation}
\label{case2-block}
\begin{split}
\|{\u}_{q_z}-{\u}_{q_{c_z}}\|_2\leq \sum_{k=z}^{c_z-1}\|\u_{q_{k+1}}-\u_{q_{k}}\|_2\leq \sum_{k=z}^{\min\{z+\hat{d}-2,Z-1\}}\|\u_{q_{k+1}}-\u_{q_{k}}\|_2.
\end{split}
\end{equation}
Combining the fact that $c_1,\dots,c_Z$ is a permutation of $1,\dots,Z$ with (\ref{case1-block}) and (\ref{case2-block}), we have
\begin{align*}
\sum_{z=1}^Z\|{\u}_{q_z}-{\u}_{q_{c_z}}\|_2\leq&2\sum_{z=1}^Z\sum_{k=z}^{\min\{z+\hat{d}-2,Z-1\}}\|\u_{q_{k+1}}-\u_{q_{k}}\|_2\\
\leq& 2\sum_{k=1}^{\hat{d}-1}\sum_{z=1}^{Z-1}\|\u_{q_{z+1}}-\u_{q_{z}}\|_2.
\end{align*}
By noticing that $\|{\u}_{q_z}-{\u}_{q_{c_z}}\|_2\leq 2R$ under Assumption \ref{assum-b2}, and combining the above inequality with \eqref{convert-from-block-path} and \eqref{block_max_d}, we have
\begin{equation}
\label{cor2-bandit-pre-fina}
    \sum_{z=1}^Z\|{\u}_{q_z}-{\u}_{q_{c_z}}\|_2\leq\min\left\{\frac{2dP_T}{K},2RZ\right\}\leq 2\sqrt{\frac{dRZP_T}{K}}=\frac{2\sqrt{dRTP_T}}{K}.
\end{equation}

Now, we are ready to derive the final bound from Theorem \ref{thm2-bandit}. For brevity, let $B=G^2+M^2$. Combining Theorem \ref{thm2-bandit} with \eqref{cor2-bandit-pre-fina}, $\delta=\sqrt{n}/{T^{1/4}}$, and $K=n\sqrt{T}$, we have
\begin{equation*}
\begin{split}
\mathbb{E}\left[\sum_{t=1}^Tf_t(\hat{\x}_t^\eta)-\sum_{t=1}^Tf_t(\u_t)\right]\leq& \frac{R^2+R{P}_T}{\eta}+\frac{\eta\gamma \sqrt{T}}{2n}+\sqrt{\gamma}\eta n\sqrt{T}G\sum_{z=1}^Zm_z\\
&+2G\sqrt{dTRP_T}+\left(3G+\frac{GR}{r}+G\sqrt{2RP_T}\right)\sqrt{n}T^{3/4}.
\end{split}
\end{equation*}
Due to $\gamma=K({nM}/{\delta})^2+(K^2-K)G^2\leq n^2BT$ and the value of $\eta$, it is easy to verify that
\begin{equation*}
\begin{split}
&\frac{R^2+R{P}_T}{\eta}+\frac{\eta\gamma \sqrt{T}}{2n}+\sqrt{\gamma}\eta n\sqrt{T}G\sum_{z=1}^Zm_z\\
\leq&(R+{P}_T)\sqrt{B}\left(\sqrt{n}T^{3/4}+n\sqrt{T\sum_{z=1}^Zm_z}\right)+\frac{R\sqrt{nB}T^{3/4}}{2}+nRG\sqrt{T\sum_{z=1}^Zm_z}\\
\leq &(R+{P}_T)\sqrt{B}(\sqrt{n}T^{3/4}+\sqrt{dT})+\frac{R\sqrt{nB}T^{3/4}}{2}+RG\sqrt{dT}
\end{split}
\end{equation*}
where the last inequality is due to \eqref{block-level-delay-bound} and $K=n\sqrt{T}$.


Finally, combining the above two inequalities, we have
\begin{equation*}
   \E\left[\sum_{t=1}^Tf_t(\hat{\x}_t^\eta)-\sum_{t=1}^Tf_t(\u_t)\right]=O\left((\sqrt{n}T^{3/4}+\sqrt{dT})(P_T+1)\right).
\end{equation*}

\subsection{Proof of Theorems \ref{thm3-bandit} and \ref{thm4-bandit}}
Following the proof of Theorems \ref{thm1-bandit} and \ref{thm2-bandit}, let $\hat{f}_{t,\delta}(\x)$ denote the $\delta$-smoothed version of $f_t(\x)$, let $q_z=(z-1)K+1$, and let $\hat{\u}_t=(1-\delta/r)\mathbf{u}_t$ denote the shrunk comparator. We first notice that \eqref{thm2-bandit-eq1} and \eqref{thm2-bandit-eq2} in the proof of Theorems \ref{thm1-bandit} and \ref{thm2-bandit} also hold when $\hat{\x}_t^\eta$ and $\x_z^\eta$ are replaced with $\hat{\x}_t$ and $\x_z$ in Algorithm \ref{Mild-BGD}, respectively. Therefore, following \eqref{thm2-bandit-eq1} and \eqref{thm2-bandit-eq2} and combining Lemma \ref{lem1-thm2-bandit}, it is not hard to verify that
\begin{equation}
\label{thm3-bandit-eq1}
\begin{split}
&\sum_{t=1}^Tf_t(\hat{\x}_t)-\sum_{t=1}^Tf_t(\u_t)\\
\leq &\sum_{z=1}^Z\sum_{t=q_z}^{q_{z+1}-1}\langle\nabla\hat{f}_{t,\delta}({\x}_z),{\x}_z-\hat{\u}_{q_z}\rangle+3\delta GT+\frac{\delta GRT}{r}+G\sqrt{2(K-1)RTP_T}.
\end{split}
\end{equation}
For brevity, let $\nabla_z=\sum_{t=q_z}^{q_{z+1}-1}\nabla\hat{f}_{t,\delta}({\x}_z)$. From \eqref{thm3-bandit-eq1}, for any $\eta\in\H$, it is easy to verify that
\begin{equation}
\label{thm3-bandit-eq2}
\begin{split}
\sum_{t=1}^Tf_t(\hat{\x}_t)-\sum_{t=1}^Tf_t(\u_t)\leq &{\sum_{z=1}^Z\langle\nabla_z,{\x}_z-{\x}_z^\eta\rangle}+{\sum_{z=1}^Z\langle\nabla_z,{\x}_z^\eta-\hat{\u}_{q_z}\rangle}\\
&+3\delta GT+\frac{\delta GRT}{r}+G\sqrt{2(K-1)RTP_T}.
\end{split}
\end{equation}
For the first term in the right side of \eqref{thm3-bandit-eq2}, we can derive the following lemma.
\begin{lem}
\label{thm3-lem1}
Let $m_z=z-1-\sum_{i=1}^{z-1}|\A_i|$ and $\nabla_z=\sum_{t=q_z}^{q_{z+1}-1}\nabla\hat{f}_{t,\delta}({\x}_z)$, where $q_z=(z-1)K+1$. Under Assumptions \ref{assum1} and \ref{assum-b2}, for any $\eta\in\H$, Algorithm \ref{Mild-BGD} ensures that
\begin{equation}
\label{lem10-eq-dis}
\begin{split}
\mathbb{E}\left[\sum_{z=1}^Z\langle\nabla_z,{\x}_z-{\x}_z^\eta\rangle\right]
\leq\frac{1}{\alpha}\ln\frac{1}{w_1^\eta}+\alpha\sqrt{\gamma}KGR^2\sum_{z=1}^Z m_{z}+\alpha \gamma ZR^2
\end{split}
\end{equation}
where $\gamma=K({nM}/{\delta})^2+(K^2-K)G^2$.
\end{lem}

To proceed the analysis, as discussed in Remark \ref{rem5:label}, we virtually introduce some additional blocks $Z+1,\dots,\lceil (T+d-1)/K\rceil$ after the final block $Z$ in Algorithm \ref{Mild-BGD}, which ensures that all the queried loss values are fully exploited to generate $\y_1^\eta,\dots,\y_{Z+1}^\eta,\forall\eta\in\H$. Moreover, we still denote the block index of the $\tau$-th utilized cumulative estimated gradients by $c_\tau$, i.e., setting $c_\tau=k$ at the beginning of step $15$ in Algorithm \ref{Mild-BGD}. Then, it is easy to verify that \eqref{thm2-bandit-eq-final-pre} in the proof of Theorems \ref{thm1-bandit} and \ref{thm2-bandit} also holds when $\sum_{t=q_z}^{q_{z+1}-1} \nabla \hat{f}_{t,\delta}({\x}_z^\eta)$ is replaced with $\nabla_z$, i.e.,
\begin{equation}
\label{thm3-bandit-eq3-pre}
\begin{split}
\mathbb{E}\left[{\sum_{z=1}^Z\langle\nabla_z,{\x}_z^\eta-\hat{\u}_{q_z}\rangle}\right]\leq& \frac{R^2+R\hat{P}_Z}{\eta}+\frac{\eta \gamma Z}{2}+\sqrt{\gamma}\eta KG\sum_{z=1}^Zm_z+\sum_{z=1}^ZKG\|\hat{\u}_{q_z}-\hat{\u}_{q_{c_z}}\|_2\\
\leq& \frac{R^2+R{P}_T}{\eta}+\frac{\eta \gamma Z}{2}+\sqrt{\gamma}\eta KG\sum_{z=1}^Zm_z+\sum_{z=1}^ZKG\|{\u}_{q_z}-{\u}_{q_{c_z}}\|_2
\end{split}
\end{equation}
where $\hat{P}_Z=\sum_{z=2}^Z\|\hat{\u}_{q_z}-\hat{\u}_{q_{z-1}}\|_2$, and the last inequality is due to \eqref{convert-path} and \eqref{convert-from-block-path}.

In the following, we derive the dynamic regret bounds in Theorems \ref{thm3-bandit} and \ref{thm4-bandit}, respectively. Note that if only considering a fixed path-length $P_T$, one can improve the dynamic regret bound in Corollary \ref{cor2-bandit:label} with the following learning rate
\begin{equation}
\label{ideal-rate-bandit}
\eta_\ast=\frac{\sqrt{R(2R+P_T)}}{\max\left\{\sqrt{n}T^{3/4},n\sqrt{T\beta^\prime}\right\}\sqrt{G^2+M^2}}.
\end{equation}
Due to Assumption \ref{assum-b2}, it is easy to verify that $0\leq P_T\leq 2(T-1)R$, and thus the set $\H=\{\eta_i=2^{i}c |i=1,\dots,N\}$ in Theorem \ref{thm3-bandit} ensures that
\begin{equation*}
\eta_1=\frac{\sqrt{2}R}{\max\left\{\sqrt{n}T^{3/4},n\sqrt{T\beta^\prime}\right\}\sqrt{G^2+M^2}}\leq\eta_\ast\leq \frac{\sqrt{2T}R}{\max\left\{\sqrt{n}T^{3/4},n\sqrt{T\beta^\prime}\right\}\sqrt{G^2+M^2}}\leq\eta_{|\H|}
\end{equation*}
because of $N=\lceil(1/2)\log_2 T \rceil+1$ and $c= R/(\sqrt{2(G^2+M^2)}\max\left\{\sqrt{n}T^{3/4},n\sqrt{T\beta^\prime}\right\})$. 

Therefore, for any possible value of $P_T$, there must exist a learning rate $\eta_k\in\H$ such that \eqref{eq3-thm1} holds, where $k=\lfloor\log_2\sqrt{(P_T+2R)/(2R)}\rfloor+1$.
Combining \eqref{thm3-bandit-eq2} with \eqref{lem10-eq-dis}, \eqref{thm3-bandit-eq3-pre}, \eqref{cor2-bandit-pre-fina}, and \eqref{eq3-thm1}, it is not hard to verify that
\begin{equation*}
\begin{split}
&\E\left[\sum_{t=1}^Tf_t(\hat{\x}_t)-\sum_{t=1}^Tf_t(\u_t)\right]\\
\leq &\frac{1}{\alpha}\ln\frac{1}{w_1^{\eta_k}}+\alpha\sqrt{\gamma}KGR^2\sum_{z=1}^Z m_{z}+\alpha \gamma ZR^2+3\delta GT+\frac{\delta GRT}{r}+G\sqrt{2(K-1)RTP_T}\\
&+\frac{2(R^2+RP_T)}{\eta_\ast}+\frac{\eta_\ast \gamma Z}{2}+\sqrt{\gamma}\eta_\ast KG\sum_{z=1}^Zm_z+2G\sqrt{{dTRP_T}}.
\end{split}
\end{equation*}
For brevity, let $B=G^2+M^2$. By substituting \eqref{ideal-rate-bandit}, $\delta=\sqrt{n}/{T^{1/4}}$, $K=n\sqrt{T}$, $\alpha= \sqrt{2}c/R^2$, and $(1/w_1^{\eta_k})\leq(k+1)^2$ into the above inequality, we can complete the proof of Theorem \ref{thm3-bandit} as below
\begin{equation*}
\begin{split}
&\E\left[\sum_{t=1}^Tf_t(\hat{\x}_t)-\sum_{t=1}^Tf_t(\u_t)\right]\\
\leq &2R\ln(k+1)\sqrt{B}\left(\sqrt{n}T^{3/4}+n\sqrt{T\beta^\prime}\right)+nGR\sqrt{T\beta^\prime}+\left(\sqrt{B}R+3G+\frac{GR}{r}\right)\sqrt{n}T^{3/4}\\
&+G\sqrt{2nRP_T}T^{3/4}+2\sqrt{BR(2R+P_T)}\left(\sqrt{n}T^{3/4}+n\sqrt{T\beta^\prime}\right)+\frac{\sqrt{nBR(2R+P_T)}T^{3/4}}{2}\\
&+nG\sqrt{R(2R+P_T)T\beta^\prime}+2G\sqrt{{dTRP_T}}\\
=&O\left((\sqrt{n}T^{3/4}+\sqrt{dT})\sqrt{P_T+1}\right)
\end{split}
\end{equation*}
where $\beta^\prime=\sum_{z=1}^Z m_{z}$ and the equality is due to $n\sqrt{T\beta^\prime}\leq \sqrt{dT}$ derived from \eqref{block-level-delay-bound} and $K=n\sqrt{T}$.

Next, to prove Theorem \ref{thm4-bandit}, we 
define the following learning rate
\begin{equation}
\label{ideal-rate-bandit2}
\eta_\ast=\frac{\sqrt{R(2R+P_T)}}{M\max\left\{\sqrt{n}T^{3/4},(n\beta T)^{1/3}\right\}}
\end{equation}
which can improve the dynamic regret bound in Corollary \ref{cor1-bandit:label} for a specific path-length $P_T$. It is also easy to verify that $\eta_1\leq \eta_\ast\leq \eta_{|\H|}$ for the set $\H=\{\eta_i=2^{i}c |i=1,\dots,N\}$ in Theorem \ref{thm4-bandit}, where $c= R/(\sqrt{2}M\max\left\{\sqrt{n}T^{3/4},(n\beta T)^{1/3}\right\})$. Therefore, for any possible value of $P_T$, there must also exist a learning rate $\eta_k\in\H$ such that \eqref{eq3-thm1} holds, where $k=\lfloor\log_2\sqrt{(P_T+2R)/(2R)}\rfloor+1$. Moreover, we notice that in the special case with $K=1$ and Assumption \ref{assum4}, the last term in the right side of \eqref{thm3-bandit-eq3-pre} actually vanishes. 
As a result, under Assumption \ref{assum4}, combining \eqref{thm3-bandit-eq2} with \eqref{lem10-eq-dis}, \eqref{thm3-bandit-eq3-pre}, $K=1$, and \eqref{eq3-thm1}, we have
\begin{equation*}
\begin{split}
\E\left[\sum_{t=1}^Tf_t(\hat{\x}_t)-\sum_{t=1}^Tf_t(\u_t)\right]\leq &\frac{1}{\alpha}\ln\frac{1}{w_1^{\eta_k}}+\frac{\alpha nMGR^2}{\delta}\sum_{t=1}^T m_{t}+\alpha \frac{n^2M^2R^2T}{\delta^2}+3\delta GT\\
&+\frac{\delta GRT}{r}+\frac{2(R^2+RP_T)}{\eta_\ast}+\frac{\eta_\ast n^2M^2 T}{2\delta^2}+\frac{nM\eta_\ast G}{\delta}\sum_{t=1}^Tm_t.
\end{split}
\end{equation*}
By substituting \eqref{ideal-rate-bandit2}, $\delta=\max\left\{\sqrt{n}/{T^{1/4}},(n\beta)^{1/3}/{T^{2/3}}\right\}$, $\alpha= \sqrt{2}c/R^2$, and $(1/w_1^{\eta_k})\leq(k+1)^2$ into the above inequality, we can complete the proof of Theorem \ref{thm4-bandit} as below
\begin{equation*}
\begin{split}
&\E\left[\sum_{t=1}^Tf_t(\hat{\x}_t)-\sum_{t=1}^Tf_t(\u_t)\right]\\
\leq &2RM\ln(k+1)\left(\sqrt{n}T^{3/4}+(n\beta T)^{1/3}\right)+GR(n\beta T)^{1/3}+MR\sqrt{n}T^{3/4}\\
&+\left(3G+\frac{GR}{r}+2\sqrt{R(2R+P_T)}M\right)\left(\sqrt{n}T^{3/4}+(n\beta T)^{1/3}\right)\\
&+\frac{\sqrt{nR(2R+P_T)}MT^{3/4}}{2}+\sqrt{R(2R+P_T)}G(nT\beta)^{1/3}\\
=&O\left((\sqrt{n}T^{3/4}+(n\bar{d})^{1/3}T^{2/3})\sqrt{P_T+1}\right)
\end{split}
\end{equation*}
where $\beta=\sum_{t=1}^Tm_t$ and the equality is due to $(n\beta T)^{1/3}\leq (n\bar{d})^{1/3}T^{2/3}$ derived from \eqref{thm01-eq6pre}.

\subsection{Proof of Lemma \ref{thm3-lem1}}
This proof is similar to the proof of Lemma \ref{lem1}, but requires some specific modifications due to the blocking update mechanism and the bandit feedback. Specifically, we first introduce some notations: $L_z^\eta=\sum_{i=1}^z\sum_{k\in\A_i}\ell_k(\x_k^\eta), 
\c_z=(L_z^\eta)_{\eta\in\H}\in\mathbb{R}^{|\H|}, \w_{z}=(w_z^\eta)_{\eta\in\H}\in\mathbb{R}^{|\H|}$, $\tilde{L}_z^\eta=\sum_{i=1}^z\ell_i(\x_i^\eta)$, and $\tilde{\c}_z=(\tilde{L}_z^\eta)_{\eta\in\H}\in\mathbb{R}^{|\H|}$. Following \eqref{lem1-eq1} in the proof of Lemma \ref{lem1}, it is easy to verify that
\begin{equation*}
\w_{z+1}=\argmin_{\w\in\Delta}\left\langle-\frac{1}{\alpha}\ln(\w_1)+\c_z,\w\right\rangle+\frac{1}{\alpha}\mathcal{R}(\w)
\end{equation*}
where $\Delta=\left\{\w\succeq\ze|\langle\w,\mathbf{1}\rangle=1\right\}$ and $\mathcal{R}(\w)=\sum_{i}w_i\ln w_i$. Moreover, let $\tilde{\w}_{1}=(\tilde{w}_{1}^\eta)_{\eta\in\H}=\w_{1}$ and 
\begin{equation}
\label{thm3-lem1-eq0}
\tilde{\w}_{z+1}=(\tilde{w}_{z+1}^\eta)_{\eta\in\H}=\argmin_{\w\in\Delta}\left\langle-\frac{1}{\alpha}\ln(\w_1)+\tilde{\c}_z,\w\right\rangle+\frac{1}{\alpha}\mathcal{R}(\w)
\end{equation}
which results in an ideal weighted average of these experts, i.e., $\tilde{\x}_z=\sum_{\eta\in\H}\tilde{w}_z^\eta\x_z^\eta$. We notice that
\begin{equation}
\label{thm3-lem1-eq1}
\begin{split}
\E\left[\sum_{z=1}^Z\langle\nabla_z,{\x}_z-{\x}_z^\eta\rangle\right]=&\E\left[\sum_{z=1}^Z\langle\nabla_z,{\x}_z-\tilde{\x}_z\rangle+\sum_{z=1}^Z\langle\nabla_z,\tilde{\x}_z-{\x}_z^\eta\rangle\right]\\
\leq &\E\left[\sum_{z=1}^ZKG\|{\x}_z-\tilde{\x}_z\|_2+\sum_{z=1}^Z\langle\nabla_z,\tilde{\x}_z-{\x}_z^\eta\rangle\right]\\
= &\E\left[\sum_{z=1}^ZKG\|{\x}_z-\tilde{\x}_z\|_2\right]+\E\left[\sum_{z=1}^Z\langle\g_z,\tilde{\x}_z-{\x}_z^\eta\rangle\right]
\end{split}
\end{equation}
where the inequality is due to the third property of Lemma \ref{estimator-lem2} and the definition of $\nabla_z$, and the last equality is due to $\g_z=\sum_{t=q_z}^{q_{z+1}-1}\frac{n}{\delta} f_t(\x_z+\delta\s_t) \mathbf{s}_t$ and Lemma \ref{estimator-lem1}.

Let $\U_z=[z]\setminus\cup_{i\in[z]}\A_i$.  For any $z>1$, combining Lemma \ref{stability} with the fact that $\mathcal{R}(\w)$ is $1$-strongly convex with respect to $\|\cdot\|_1$, we have
\begin{equation}
\label{thm3-lem1-eq2}
\begin{split}
\left\|\tilde{\x}_z-\x_z\right\|_2=&\left\|\sum_{\eta\in\H}(\tilde{w}_z^\eta-w_z^\eta)\x_z^\eta\right\|_2\leq\sum_{\eta\in\H}|\tilde{w}_z^\eta-w_z^\eta|\left\|\x_z^\eta\right\|_2\leq R\|\tilde{\w}_z-\w_z\|_1\\
\leq&\alpha R\|\tilde{\c}_{z-1}-\c_{z-1}\|_\infty\leq\alpha R\max_{\eta\in\H}\left|\sum_{k\in\U_{z-1}}\ell_k(\x_k^\eta)\right|\leq \alpha R^2\sum_{k\in\U_{z-1}}\|\g_k\|_2
\end{split}
\end{equation}
where the second inequality is due to Assumption \ref{assum-b2}, and the last inequality is due to 
\begin{equation}
\label{thm3-lem1-eq3}
|\ell_k(\x_k^\eta)|=|\langle\g_k,\x_k^\eta\rangle|\leq \|\g_k\|_2\|\x_k^\eta\|_2\leq \|\g_k\|_2R,\forall k\in[Z],\eta\in\H.
\end{equation}
Note that the upper bound provided by Lemma \ref{block-gradient} also holds for $\g_z$ in Algorithm \ref{Mild-BGD}. Therefore, combining \eqref{thm3-lem1-eq2} with Lemma \ref{block-gradient} and Jensen's inequality, we have
\begin{equation}
\label{thm3-lem1-eq4}
\begin{split}
\E\left[\left\|\tilde{\x}_z-\x_z\right\|_2\right]\leq \alpha |\U_{z-1}|\sqrt{\gamma}R^2=\alpha\sqrt{\gamma}R^2m_{z-1}
\end{split}
\end{equation}
where $\gamma=K({nM}/{\delta})^2+(K^2-K)G^2$, and the equality is due to the definition of $\U_{z-1}$.

Then, we proceed to analyze the last term in the right side of \eqref{thm3-lem1-eq1}. 
Let $\mathbf{d}_z=(\ell_z(\x_z^\eta))_{\eta\in\H}\in\mathbb{R}^{|\H|}$ and $h_z(\x)=\langle\x,\mathbf{d}_z\rangle$ for any $z\in[Z]$. Additionally, let $h_0(\x)=(\mathcal{R}(\w)-\langle\ln(\w_1),\w\rangle)/{\alpha}$ and $\Delta^\prime=\left\{\mathbf{e}_1,\dots,\mathbf{e}_{|\H|}\right\}\subset\Delta$, where $\mathbf{e}_i\in\mathbb{R}^{|\H|}$ and its $i$-th element equals to $1$. Combining with the previous definitions, it is not hard to verify that
\begin{equation}
\label{thm3-lem1-eq5}
\begin{split}
\sum_{z=1}^Z\langle\g_z,\tilde{\x}_z-{\x}_z^\eta\rangle=&\sum_{z=1}^Z\left(\langle\tilde{\w}_{z},\mathbf{d}_z\rangle-\ell_z(\x_z^\eta)\right)\leq \sum_{z=1}^Zh_z(\tilde{\w}_{z})-\sum_{z=1}^Zh_z(\mathbf{e}_{\eta})\\
\leq&\sum_{z=1}^Z\left(h_z(\tilde{\w}_{z+1})-h_z(\mathbf{e}_{\eta})\right)+\sum_{z=1}^Z\left(h_z(\tilde{\w}_{z})-h_z(\tilde{\w}_{z+1})\right).
\end{split}
\end{equation}
Moreover, combining with \eqref{thm3-lem1-eq0} and $\tilde{\w}_1=\w_1$, we also have $\tilde{\w}_{z+1}=\argmin_{\w\in\Delta}\sum_{i=0}^{z}h_i(\w)$ for any $z\geq 0$. Therefore, we can apply Lemma \ref{lem-ftl} to derive that
\begin{equation}
\label{thm3-lem1-eq6}
\begin{split}
\sum_{z=1}^Z\left(h_z(\tilde{\w}_{z+1})-h_z(\mathbf{e}_{\eta})\right)\leq h_0(\mathbf{e}_{\eta})-h_0(\tilde{\w}_{1})=\frac{\langle\ln(\mathbf{e}_{\eta})-\ln(\w_1),\mathbf{e}_{\eta}\rangle}{\alpha}=\frac{1}{\alpha}\ln\frac{1}{w_1^\eta}.
\end{split}
\end{equation}
For the last term in the right side of \eqref{thm3-lem1-eq5}, we have
\begin{equation}
\label{thm3-lem1-eq7}
\begin{split}
&h_z(\tilde{\w}_{z})-h_z(\tilde{\w}_{z+1})=\langle\tilde{\w}_{z}-\tilde{\w}_{z+1},\mathbf{d}_z\rangle\leq\|\tilde{\w}_{z}-\tilde{\w}_{z+1}\|_1\|\mathbf{d}_z\|_{\infty}\\
\leq&\alpha \|\tilde{\c}_{z-1}-\tilde{\c}_{z}\|_\infty\|\mathbf{d}_z\|_{\infty}=\alpha \|\mathbf{d}_z\|_{\infty}^2=\alpha \left(\max_{\eta\in\H}\left|\ell_z(\x_z^\eta)\right|\right)^2\leq \alpha R^2\|\g_z\|_2^2
\end{split}
\end{equation}
where the first inequality is due to Hölder's inequality, the second inequality is due to Lemma \ref{stability} and \eqref{thm3-lem1-eq0}, and the last inequality is due to \eqref{thm3-lem1-eq3}. By substituting \eqref{thm3-lem1-eq6} and \eqref{thm3-lem1-eq7} into \eqref{thm3-lem1-eq5}, we have
\begin{equation}
\label{thm3-lem1-eq8}
\begin{split}
\sum_{z=1}^Z\langle\g_z,\tilde{\x}_z-{\x}_z^\eta\rangle\leq \frac{1}{\alpha}\ln\frac{1}{w_1^\eta}+\alpha R^2\|\g_z\|_2^2.
\end{split}
\end{equation}
Finally, by substituting \eqref{thm3-lem1-eq4} and \eqref{thm3-lem1-eq8} into \eqref{thm3-lem1-eq1}, we have 
\begin{equation*}
\begin{split}
\E\left[\sum_{z=1}^Z\langle\nabla_z,{\x}_z-{\x}_z^\eta\rangle\right]\leq& \alpha\sqrt{\gamma}KGR^2\sum_{z=1}^Z m_{z}+\frac{1}{\alpha}\ln\frac{1}{w_1^\eta}+\sum_{z=1}^Z\E\left[\alpha R^2\|\g_z\|_2^2\right]\\
\leq &\alpha\sqrt{\gamma}KGR^2\sum_{z=1}^Z m_{z}+\frac{1}{\alpha}\ln\frac{1}{w_1^\eta}+\alpha \gamma ZR^2
\end{split}
\end{equation*}
where the last inequality is due to Lemma \ref{block-gradient}.

\section{Conclusion}
In this paper, we study the dynamic regret of OCO with arbitrary delays, and consider both the full-information and bandit settings. First, for the full-information setting, we propose a novel algorithm called Mild-OGD, whose dynamic regret can be automatically bounded by $O(\sqrt{\bar{d}T(P_T+1)})$ under the in-order assumption and $O(\sqrt{dT(P_T+1)})$ in the worst case. A matching lower bound has also been provided to show the optimality of Mild-OGD in the worst case. Second, we propose a variant of Mild-OGD to handle the bandit setting, namely Mild-BGD. It achieves a dynamic regret bound of $O((\sqrt{n}T^{3/4}+(n\bar{d})^{1/3}T^{2/3})\sqrt{P_T+1})$ under the in-order assumption, and a dynamic regret bound of $O((\sqrt{n}T^{3/4}+\sqrt{dT})\sqrt{P_T+1})$ in the worst case. These bounds can match the existing dynamic regret bound in the non-delayed bandit setting under a relatively large amount of delay.
\appendix
\section{Mild-OGD with the Doubling Trick}
\label{appendixA}
As discussed in Remark \ref{rem4}, our Mild-OGD needs to set $\alpha$ and $\eta_1,\dots,\eta_N$
based on on the following value \[
\sum_{t=1}^T(m_t+1)=\sum_{t=1}^T\left(t-\sum_{i=1}^{t-1}|\F_i|\right).\]
However, it may be not available beforehand. Fortunately, the doubling trick \citep{P_book_2006} provides a way to adaptively estimate this value. Specifically, it will divide the total $T$ rounds into several epochs, and run a new instance of Mild-OGD per epoch. Let $s_v$ and $s_{v+1}-1$ respectively denote the start round and the end round of the $v$-th epoch. In this way, to tune the parameters for the $v$-th epoch, we only need to know the following value
\[
\sum_{t=s_v}^{s_{v+1}-1}\left(t+1-s_v-\sum_{i=s_v}^{t-1}|\F_i^{s_v}|\right)
\]
where $\F_i^{s_v}=\{k\in[s_v,i]|k+d_k-1=i\}$.

According to the doubling trick, we can estimate this value to be $2^v$ at the start round $s_v$ of the $v$-th epoch. Then, for any $t>s_v$, we first judge whether the estimate is still valid, i.e.,
\[
\sum_{j=s_v}^{t}\left(j+1-s_v-\sum_{i=s_v}^{j-1}|\F_i^{s_v}|\right)\leq 2^v
\]
where the left side can be calculated at the beginning of round $t$. If the answer is positive, the round $t$ is still assigned to the $v$-th epoch, and the instance of Mild-OGD keeps running. Otherwise, the round $t$ is set as the start round of the $(v+1)$-th epoch, and a new instance of Mild-OGD is activated. Note that in the start round of the $(v+1)$-th epoch, the new estimate must be valid, since $t=s_{v+1}$ and
\[
\sum_{j=s_{v+1}}^{t}\left(j+1-s_{v+1}-\sum_{i=s_{v+1}}^{j-1}|\F_i^{s_{v+1}}|\right)=1\leq 2^{v+1}.
\]
Moreover, it is natural to set $s_1=1$. 

Based on the above discussions, the detailed procedure of Mild-OGD with the doubling trick is summarized in Algorithm \ref{Ader-Meta-DT}, where the meta-algorithm and expert-algorithm are integrated into one for brevity.
\begin{remark}
\emph{
To facilitate presentation, in step 2 of Algorithm \ref{Ader-Meta-DT}, each $\eta_i$ in $\H$ only contains the constant part that does not depend on the value of $\sum_{t=1}^T(m_t+1)$. The \emph{if} condition in step $5$ of Algorithm \ref{Ader-Meta-DT} determines whether to start a new epoch. Moreover, as shown in step $6$ of Algorithm \ref{Ader-Meta-DT}, in the start of each epoch, we need to reinitialize the weight and the decision of each expert $E^\eta$. For the update of weights and decisions, we notice that in step $9$ of Algorithm \ref{Ader-Meta-DT} only $\{\nabla f_k(\x_k)|k\in\F_t^{s_v}\}$ is received, instead of $\{\nabla f_k(\x_k)|k\in\F_t\}$, because we do not need to utilize gradients queried before the current epoch $v$. In step $10$ of Algorithm \ref{Ader-Meta-DT}, we update the weight during each epoch $v$ by using $\alpha_v=1/(GR2^{v/2})$ computed based on the estimation of $\sum_{t=1}^T(m_t+1)$, i.e.,
\begin{equation}
\label{weight-double}
    w_{t+1}^{\eta}=\frac{w_t^{\eta} e^{-\alpha_v\sum_{k\in\F_t^{s_v}}\ell_{k}(\x_k^{\eta})}}{\sum_{\mu\in\H}w_t^{\mu}e^{-\alpha_v\sum_{k\in\F_t^{s_v}}\ell_{k}(\x_k^\mu)}}
\end{equation}
where $\ell_k(\x)=\langle\nabla f_k(\x_k),\x\rangle$. Meanwhile, in step $12$ of Algorithm \ref{Ader-Meta-DT}, each expert $E^\eta$ updates its decision with the learning rate derived by combining the constant $\eta\in\H$ with the estimation of $\sum_{t=1}^T(m_t+1)$, i.e.,
\begin{equation}
\label{decision-double}
\y_{\tau+1}^\eta=\argmin_{\x\in\K}\left\|\x-\left(\y_{\tau}^\eta-\frac{\eta }{2^{v/2}}\nabla f_k(\x_k)\right)\right\|_2^2.
\end{equation}
}
\end{remark}

We have the following theorem, which can recover the dynamic regret bound in Theorem \ref{thm2} up to a constant factor.
\begin{algorithm}[t]
\caption{Mild-OGD with the Doubling Trick}
\label{Ader-Meta-DT}
\begin{algorithmic}[1]
\STATE \textbf{Initialization:} set $v=1$, $s_v=1$, and $\tau=1$
\STATE Let $N=\lceil(1/2)\log_2T\rceil+1$ and $\H=\left\{\eta_i=\left.R2^{i}/(\sqrt{2}G)\right|i=1,\dots,N\right\}$
\STATE Activate experts $\left\{E^{\eta}|\eta\in\H\right\}$ by initializing $\y_1^\eta=\ze,\forall\eta\in\H$, and $w_{1}^{\eta_i}=\frac{|\H|+1}{i(i+1)|\H|},\forall i\in[N]$
\FOR{$t=1,\dots,T$}
\IF{$\sum_{j=s_v}^t\left(j+1-s_v-\sum_{i=s_v}^{j-1}|\F_i^{s_v}|\right)>2^{v}$}
\STATE Set $v=v+1$, $s_v=t$, $\tau=1$, $\y_1^\eta=\ze,\forall\eta\in\H$, and $w_{t}^{\eta_i}=\frac{|\H|+1}{i(i+1)|\H|},\forall i\in[N]$
\ENDIF
\STATE Play the decision $\x_t=\sum_{\eta\in\H}w_{t}^\eta\x_t^\eta$, where $\x_t^\eta=\y_{\tau}^\eta$ for each expert $E^\eta$ 
\STATE Query $\nabla f_t(\x_t)$, and receive $\{\nabla f_k(\x_k)|k\in\F_t^{s_v}\}$ 
\STATE Update the weight of each expert $E^\eta$ as in \eqref{weight-double}, where $\alpha_v=1/(GR2^{v/2})$
\FOR{$k\in\F_t$ (in the ascending order)}
\STATE Update the decision of each expert $E^\eta$ as in \eqref{decision-double}, and set $\tau=\tau+1$
\ENDFOR
\ENDFOR
\end{algorithmic}
\end{algorithm}
\begin{theorem}
\label{thm4}
Under Assumptions \ref{assum1} and \ref{assum2}, for any comparator sequence $\u_1,\dots,\u_T\in\K$, Algorithm \ref{Ader-Meta-DT} ensures
\[
\sum_{t=1}^Tf_t(\x_t)-\sum_{t=1}^Tf_t(\u_t)\leq\frac{2\left(\left(2\ln\left(k+1\right)+1\right)GR+3G\sqrt{R(2R+P_T)}\right)\sqrt{\bar{d}T}}{\sqrt{2}-1}+C
\]
where $C$ is defined in (\ref{bigC}) and $k=\lfloor\log_2\sqrt{(P_T+2R)/2R}\rfloor+1$.
\end{theorem}
\begin{proof}
For any $s_v$ and $j\geq s_v$, we notice that the value of $j-s_v-\sum_{i=s_v}^{j-1}|\F_i^{s_v}|$ counts the number of gradients that have been queried over interval $[s_v,j-1]$, but still not arrive at the end of round $j-1$. Moreover, the gradient $\nabla f_j(\x_j)$ will only be counted as an unreceived gradient in $d_j-1$ rounds. Therefore, for any $s_v\leq t\leq T$, it is easy to verify that
\begin{align*}
\sum_{j=s_v}^{t}\left(j+1-s_v-\sum_{i=s_v}^{j-1}|\F_i^{s_v}|\right)\leq\sum_{j=s_v}^td_j\leq\sum_{j=1}^Td_j=\bar{d}T.
\end{align*}
For brevity, let $V$ denote the final $v$ of Algorithm \ref{Ader-Meta-DT}, and let $S=\bar{d}T,s_{V+1}=T+1$. It is easy to verify that
\begin{equation}
\label{eq0-DT}
V\leq1+\log_2 S.
\end{equation}
Now, we consider the dynamic regret of Algorithm \ref{Ader-Meta-DT} over the interval $[s_v,s_{v+1}-1]$ for each $v\in[V]$. Let $P_{s_v}^{s_{v+1}}=\sum_{t=s_{v}+1}^{s_{v+1}-1}\|\u_t-\u_{t-1}\|_2$ and $\eta_\ast^v=\sqrt{R(2R+P_{s_v}^{s_{v+1}})}/G$. From Assumption \ref{assum2}, we have 
\begin{equation*}
\eta_1=\frac{\sqrt{2}R}{G}\leq\eta_\ast^v\leq \sqrt{\frac{2R^2(s_{v+1}-s_{v})}{G^2}}\leq\frac{\sqrt{2T}R}{G}\leq\eta_{|\H|}.
\end{equation*}
Therefore, for any possible value of $P_{s_v}^{s_{v+1}}$, there must exist a constant $\eta_{k_v}\in\H$ such that
\begin{equation}
\label{eq3-thm4}
\eta_{k_v}\leq\eta_\ast^v\leq2\eta_{k_v}
\end{equation}
where $k_v=\lfloor\log_2\sqrt{(2R+P_{s_v}^{s_{v+1}})/(2R)}\rfloor+1\leq k$.

Note that each expert $E^\eta$ over the interval $[s_v,s_{v+1}-1]$ actually runs Algorithm \ref{Ader-Expert} with the learning rate $\frac{\eta}{2^{v/2}}$ to handle the surrogate losses $\ell_{s_v}(\x),\dots,\ell_{s_{v+1}-1}(\x)$, where each gradient $\nabla \ell_t(\x_t^\eta)=\nabla f_t(\x_t)$ is delayed to the end of round $t+d_t-1$ for $t\in[s_v,s_{v+1}-1]$. Therefore, combining Theorem \ref{thm1} with Lemma \ref{lem1_DOGD_main}, under Assumptions \ref{assum1} and \ref{assum2}, we have
\begin{equation}
\label{eq4-thm4-pre}
\begin{split}
&\sum_{t=s_v}^{s_{v+1}-1}\left(\ell_t(\x_t^{\eta_{k_v}})-\ell_t(\u_t)\right)\\
\leq&
\frac{2^{v/2}\left(R^2+RP_{s_v}^{s_{v+1}}\right)}{\eta_{k_v}}+\frac{\eta_{k_v}G^2}{2^{v/2}}\sum_{j=s_v}^{s_{v+1}-1}\left(j+1-s_v-\sum_{i=s_v}^{j-1}|\F_i^{s_v}|\right)+C_v
\end{split}
\end{equation}
where $C_v$ is defined as
\begin{equation*}
C_v=\left\{
\begin{aligned}
&0,~\text{if Assumption \ref{assum4} also holds;}\\
&\min\left\{2(s_{v+1}-s_v)GR,2dGP_{s_v}^{s_{v+1}}\right\},~\text{otherwise.}
\end{aligned}
\right.
\end{equation*}
Moreover, according to Algorithm \ref{Ader-Meta-DT}, we have
\begin{equation}
\label{DT-condition}
\sum_{j=s_v}^{s_{v+1}-1}\left(j+1-s_v-\sum_{i=s_v}^{j-1}|\F_i^{s_v}|\right)\leq 2^v.
\end{equation}
By substituting \eqref{DT-condition} into \eqref{eq4-thm4-pre}, we have
\begin{equation}
\label{eq4-thm4}
\begin{split}
\sum_{t=s_v}^{s_{v+1}-1}\left(\ell_t(\x_t^{\eta_{k_v}})-\ell_t(\u_t)\right)
\leq&
\frac{2^{v/2}(R^2+RP_{s_v}^{s_{v+1}})}{\eta_{k_v}}+\eta_{k_v}G^22^{v/2}+C_v\\
\leq&
3G\sqrt{2^{v}R(2R+P_{s_v}^{s_{v+1}})}+C_v\\
\leq&3G\sqrt{2^{v}R(2R+P_T)}+C_v
\end{split}
\end{equation}
where the second inequality is due to (\ref{eq3-thm4}) and the definition of $\eta_\ast^v$.

Next, it is also easy to verify that Algorithm \ref{Ader-Meta-DT} actually starts or restarts Algorithm \ref{Ader-Meta} with the learning rate of $\alpha_v$ at round $s_v$, which ends at round $s_{v+1}-1$. By using Lemma \ref{lem1} with $(1/w_{s_v}^{\eta_{k_v}})\leq(k_v+1)^2$, (\ref{DT-condition}), and $\alpha_v=1/(GR2^{v/2})$, under Assumptions \ref{assum1} and \ref{assum2},
we have
\begin{equation}
\label{eq5-thm4}
\begin{split}
\sum_{t=s_v}^{s_{v+1}-1}\left(\ell_t\left(\x_t\right)-\ell_t(\x_t^{\eta_{k_v}})\right)\leq&\frac{2}{\alpha_v}\ln(k_v+1)+\alpha_v 2^vG^2R^2\\
\leq&\left(2\ln\left(k+1\right)+1\right)2^{v/2}GR.
\end{split}
\end{equation}
Combining (\ref{eq4-thm4}) and (\ref{eq5-thm4}), we have
\begin{equation}
\label{eq6-thm4}
\begin{split}
&\sum_{t=s_v}^{s_{v+1}-1}f_t(\x_t)-\sum_{t=s_v}^{s_{v+1}-1}f_t(\u_t)\leq\sum_{t=s_v}^{s_{v+1}-1}\ell_t(\x_t)-\sum_{t=s_v}^{s_{v+1}-1}\ell_t(\u_t)\\
=&\sum_{t=s_v}^{s_{v+1}-1}\ell_t(\x_t)-\sum_{t=s_v}^{s_{v+1}-1}\ell_t(\x_t^{\eta_{k_v}})+\sum_{t=s_v}^{s_{v+1}-1}\ell_t(\x_t^{\eta_{k_v}})-\sum_{t=s_v}^{s_{v+1}-1}\ell_t(\u_t)\\
\leq&\left(2\ln\left(k+1\right)+1\right)2^{v/2}GR+3G\sqrt{2^{v}R(2R+P_T)}+C_v.
\end{split}
\end{equation}
According to (\ref{eq0-DT}) and \eqref{eq6-thm4}, we further have
\begin{equation*}
\begin{split}
&\sum_{t=1}^Tf_t(\x_t)-\sum_{t=1}^Tf_t(\u_t)=\sum_{v=1}^V\left(\sum_{t=s_v}^{s_{v+1}-1}f_t(\x_t)-\sum_{t=s_v}^{s_{v+1}-1}f_t(\u_t)\right)\\
\leq&\left(\left(2\ln\left(k+1\right)+1\right)GR+3G\sqrt{R(2R+P_T)}\right)\sum_{v=1}^V2^{v/2}+\sum_{v=1}^VC_v\\
=&\left(\left(2\ln\left(k+1\right)+1\right)GR+3G\sqrt{R(2R+P_T)}\right)\frac{\sqrt{2}(2^{V/2}-1)}{\sqrt{2}-1}+\sum_{v=1}^VC_v\\
\leq&\frac{2\left(\left(2\ln\left(k+1\right)+1\right)GR+3G\sqrt{R(2R+P_T)}\right)\sqrt{S}}{\sqrt{2}-1}+\sum_{v=1}^VC_v.
\end{split}
\end{equation*}
Finally, this proof is completed by combining the above inequality with $S=\bar{d}T$ and $\sum_{v=1}^VC_v\leq C$, where the latter is derived from
\begin{align*}
\sum_{v=1}^V\min\left\{2(s_{v+1}-s_v)GR,2dGP_{s_v}^{s_{v+1}}\right\}\leq&\min\left\{\sum_{v=1}^V2(s_{v+1}-s_v)GR,\sum_{v=1}^V2dGP_{s_v}^{s_{v+1}}\right\}\\
\leq&\min\left\{2TGR,2dGP_T\right\}.
\end{align*}
\end{proof}
\section{Additional Results of Algorithm \ref{BDBGD} for a Specific Path-length}
\label{appendixB}
Corollaries \ref{cor1-bandit:label} and \ref{cor2-bandit:label} have provided dynamic regret bounds for Algorithm \ref{BDBGD} with respect to any comparator sequence. Interestingly, if only focusing on a specific $P_T$, these bounds actually can be further improved by exploiting $P_T$ to tune the parameters $\delta$ and $\eta$. Specifically, we can derive the following corollaries from Theorems \ref{thm1-bandit} and \ref{thm2-bandit}.
\begin{cor}
\label{cor3-bandit:label}
Let $c= R^2+RP_T$, and $\beta=\sum_{t=1}^Tm_t$, where $m_t=t-1-\sum_{i=1}^{t-1}|\F_i|$.~Under Assumptions \ref{assum1}, \ref{assum-b1}, \ref{assum-b2}, and \ref{assum4}, for a specific sequence of comparators $\u_1,\dots,\u_T\in\K$, Algorithm \ref{BDBGD} with $\eta=\min\left\{c^{3/4}/(M\sqrt{n}T^{3/4}),c^{2/3}/(M(n\beta T)^{1/3})\right\}$, $\delta=\max\left\{\sqrt{n}c^{1/4}/{T^{1/4}},(nc\beta)^{1/3}/{T^{2/3}}\right\}\in(0,r)$, and $K=1$ ensures
\begin{equation*}
   \E\left[\sum_{t=1}^Tf_t(\hat{\x}_t^\eta)-\sum_{t=1}^Tf_t(\u_t)\right]=O\left(\sqrt{n}T^{3/4}(P_T+1)^{1/4}+(n\bar{d})^{1/3}T^{2/3}(P_T+1)^{1/3}\right).
\end{equation*}
\end{cor}
\begin{proof}
Due to the values of $\eta$, $\delta$, $\beta$, and $c$, it is easy to verify that
\begin{equation*}
\begin{split}
&\frac{R^2+RP_T}{\eta}+\frac{\eta Tn^2M^2}{2\delta^2}+\frac{\eta nGM}{\delta}\sum_{t=1}^Tm_t\\
\leq&M\left(\sqrt{n}c^{1/4}T^{3/4}+(nc\beta T)^{1/3}\right)+\frac{\sqrt{n}c^{1/4}T^{3/4}M}{2}+\left(nc\beta T\right)^{1/3}G.
\end{split}
\end{equation*}
Moreover, due to the value of $\delta$ and Lemma \ref{lem1_DOGD_main}, under Assumption \ref{assum4}, we have
\begin{equation*}
\begin{split}
\sum_{t=1}^TG\left\|{\u}_t-{\u}_{c_t}\right\|_2+
3\delta GT+\frac{\delta GRT}{r}\leq&\left(3G+\frac{GR}{r}\right)\left(\sqrt{n}c^{1/4}T^{3/4}+\left(nc\beta T\right)^{1/3}\right).
\end{split}
\end{equation*}
Finally, we can complete this proof by combining Theorem \ref{thm1-bandit} with the above two inequalities and $(nc\beta T)^{1/3}\leq (nc\bar{d})^{1/3}T^{2/3}$ derived from \eqref{thm01-eq6pre}.



\end{proof}

\begin{cor}
\label{cor4-bandit:label}
Let $c= R^2+RP_T$, and $\beta^\prime=\sum_{z=1}^Zm_z$, where $m_z=z-1-\sum_{i=1}^{z-1}|\mathcal{A}_i|$. Under Assumptions \ref{assum1}, \ref{assum-b1}, and \ref{assum-b2}, for a specific sequence of comparators $\u_1,\dots,\u_T\in\K$, Algorithm \ref{BDBGD} with $\delta=\sqrt{n}c^{1/4}/{T^{1/4}}\in(0,r)$, $\eta=\min\left\{c^{3/4}/(\sqrt{G^2+M^2}\sqrt{n}T^{3/4}),c/(\sqrt{G^2+M^2}n\sqrt{T\beta^\prime})\right\}$, and $K=n\sqrt{T/c}$ ensures
\[
    \E\left[\sum_{t=1}^Tf_t(\hat{\x}_t^\eta)-\sum_{t=1}^Tf_t(\u_t)\right]=O\left(\sqrt{n}T^{3/4}(P_T+1)^{1/4}+\sqrt{dT(P_T+1)}\right).
\]
\end{cor}
\begin{proof}
For brevity, let $B=G^2+M^2$. Combining Theorem \ref{thm2-bandit} with \eqref{cor2-bandit-pre-fina}, $\delta=\sqrt{n}c^{1/4}/{T^{1/4}}$, $K=n\sqrt{T/c}$, and the values of $c$ and $\beta^\prime$, we have
\begin{equation*}
\begin{split}
\mathbb{E}\left[\sum_{t=1}^Tf_t(\hat{\x}_t^\eta)-\sum_{t=1}^Tf_t(\u_t)\right]
\leq& \frac{c}{\eta}+\frac{\eta nBT^{3/2}}{2\sqrt{c}}+\frac{\eta n^2TG\beta^\prime\sqrt{B}}{c}+2G\sqrt{dcT}\\
&+\left(3G+\frac{GR}{r}\right)\sqrt{n}c^{1/4}T^{3/4}+G\sqrt{2n}c^{1/4}T^{3/4}.
\end{split}
\end{equation*}
Moreover, due to the value of $\eta$, it is not hard to verify that
\begin{equation*}
\begin{split}
&\frac{c}{\eta}+\frac{\eta nBT^{3/2}}{2\sqrt{c}}+\frac{\eta n^2TG\beta^\prime\sqrt{B}}{c}\\
\leq&\sqrt{B}\left(c^{1/4}\sqrt{n}T^{3/4}+n\sqrt{T\beta^\prime}\right)+\frac{c^{1/4}T^{3/4}\sqrt{nB}}{2}+nG\sqrt{T\beta^\prime}\\
\leq&\sqrt{B}\left(c^{1/4}\sqrt{n}T^{3/4}+\sqrt{dcT}\right)+\frac{c^{1/4}T^{3/4}\sqrt{nB}}{2}+G\sqrt{dcT}
\end{split}
\end{equation*}
where the last equality is due to $n\sqrt{T\beta^\prime}\leq \sqrt{dcT}$ derived from \eqref{block-level-delay-bound} and $K=n\sqrt{T/c}$.

Finally, it is easy to complete this proof by combining the above two inequalities and $c= R^2+RP_T$.

\end{proof}
\begin{remark}
\emph{
For a specific $P_T$, the two bounds in Corollaries \ref{cor3-bandit:label} and \ref{cor4-bandit:label} actually are even tighter than those bounds in Theorems \ref{thm4-bandit} and \ref{thm3-bandit}. Thus, a more ambitious goal is to extended these improved results to any comparator sequence. Note that our Mild-OGD and Mild-BGD have shown the power of the two-level framework for adaptively tuning the parameter $\eta$ according to $P_T$. At first glance, it seems that this goal can be achieved by using the two-level framework to adaptively tune both the parameters $\eta$ and $\delta$. However, we want to emphasize that this solution is impractical, because the bandit setting does not allow us to run multiple experts with different values of $\delta$. This limitation explains why we only achieve the results in Theorems \ref{thm3-bandit} and \ref{thm4-bandit}, and implies that more advanced techniques may be required, which will be investigated in the future.
}
\end{remark}

\bibliography{ref}

\end{document}